\documentclass[sigconf]{acmart}

\usepackage{booktabs}
\usepackage{microtype}
\usepackage{graphicx}
\usepackage{subfigure}

\usepackage{hyperref}

\usepackage{amsmath,amsfonts,bm}

\def\eqref#1{equation~\ref{#1}}
\def\1{\bm{1}}

\def\mS{{\bm{S}}}

\def\mX{{\bm{X}}}

\DeclareMathAlphabet{\mathsfit}{\encodingdefault}{\sfdefault}{m}{sl}
\SetMathAlphabet{\mathsfit}{bold}{\encodingdefault}{\sfdefault}{bx}{n}
\newcommand{\tens}[1]{\bm{\mathsfit{#1}}}

\def\tM{{\tens{M}}}

\def\tS{{\tens{S}}}

\def\tV{{\tens{V}}}

\usepackage{amsfonts} 
\usepackage{amsmath}
\usepackage{amsthm}
\usepackage{multicol}
\usepackage{multirow}
\usepackage{tabulary}
\usepackage{enumitem}
\usepackage{balance}
\usepackage{sidecap}
\usepackage{url}
\usepackage{algorithm}
\usepackage{algorithmic}

\usepackage{natbib}
\newcommand{\pluseq}{\mathrel{+}=}
\newcommand{\norm}[1]{\left\lVert#1\right\rVert}
\newcommand{\abs}[1]{\left|#1\right|}
\newtheorem{theorem}{Theorem}[section]
\newtheorem{lemma}{Lemma}[section]

\setcopyright{none}
\acmDOI{}

\acmISBN{}

\acmYear{2019}
\copyrightyear{2019}

\begin{document}
\title{Compressing Gradient Optimizers via Count-Sketches}

\author{Ryan Spring}
\affiliation{%
  \institution{Rice University}
  \institution{Department of Computer Science}
  \city{Houston} 
  \state{TX}
  \country{USA}
}
\email{rdspring1@rice.edu}

\author{Anastasios Kyrillidis}
\affiliation{%
  \institution{Rice University}
  \institution{Department of Computer Science}
  \city{Houston} 
  \state{TX}
  \country{USA}
}
\email{anastasios@rice.edu}

\author{Vijai Mohan}
\affiliation{%
  \institution{Amazon Search}
  \city{Palo Alto} 
  \state{CA}
  \country{USA}
}
\email{vijaim@amazon.com}

\author{Anshumali Shrivastava}
\affiliation{%
  \institution{Rice University}
  \institution{Department of Computer Science}
  \city{Houston} 
  \state{TX}
  \country{USA}
}
\email{anshumali@rice.edu}

\begin{abstract}
Many popular first-order optimization methods (e.g., Momentum, AdaGrad, Adam) accelerate the convergence rate of deep learning models. However, these algorithms require auxiliary parameters, which cost additional memory proportional to the number of parameters in the model. The problem is becoming more severe as deep learning models continue to grow larger in order to learn from complex, large-scale datasets. Our proposed solution is to maintain a linear sketch to compress the auxiliary variables. We demonstrate that our technique has the same performance as the full-sized baseline, while using significantly less space for the auxiliary variables. Theoretically, we prove that count-sketch optimization maintains the SGD convergence rate, while gracefully reducing memory usage for large-models. On the large-scale 1-Billion Word dataset, we save 25\% of the memory used during training (8.6 GB instead of 11.7 GB) by compressing the Adam optimizer in the Embedding and Softmax layers with negligible accuracy and performance loss. For an Amazon extreme classification task with over 49.5 million classes, we also reduce the training time by 38\%, by increasing the mini-batch size 3.5$\times$ using our count-sketch optimizer.
\end{abstract}


\maketitle

\section{Introduction} 
An emerging trend in natural language processing is to train a language model in an unsupervised fashion on a large corpus of text, and then to fine-tune the model for a specific task \citep{radford2018improving, puri2018largescale, devlin2018bert}. The language model often takes the form of an LSTM \citep{jozefowicz2016limits} or a Transformer network \citep{attentionNIPS2017}. 

These models already contain millions of parameters and will continue to grow even larger. Recently, \cite{yang2018breaking} demonstrated that the expressiveness of a single Softmax layer was insufficient for the language modeling task. Their proposed solution was the Mixture of Softmax (MoS) layer, which combines several independent Softmax layers together. The number of Softmax layers typically ranges between 3 and 15, so the proposed solution requires significantly more space, especially for larger vocabularies.

% Recommendation Models - Extreme Classification - Large Softmax Layer - Mach

Training large-scale models efficiently is a challenging task. There are numerous publications that describe how to leverage multi-GPU data parallelism and mixed precision training effectively \citep{largebatchNIPS2017, ott2018scaling, micikevicius2018mixed}. A key tool for improving training time is to increase the batch size, taking advantage of the massive parallelism provided by GPUs. However, increasing the batch size also requires significant amounts of memory. Often times, a practitioner will sacrifice their batch size for a larger, more expressive model. For example, \citep{puri2018largescale} showed that doubling the dimensionality of an multiplicative LSTM \citep{krause2016mlstm} from 4096 to 8192 forces them to reduce the batch size per GPU by $4\times$.

One culprit that aggravates the memory capacity issue is the auxiliary parameters used by first-order optimization algorithms, which are commonly used to accelerate the convergence rate of the model. Our proposed solution is to compress the auxiliary parameters of the optimizer using the Count-Sketch dataset structure \cite{charikar2002finding}, freeing up memory for either a more expressive model or a larger batch size for faster training.

We primarily focus on compressing the auxiliary variables for the embedding and Softmax layers. These layers contain a significant portion of the model's parameters and the set of active features or classes is extremely sparse for many tasks~\cite{spring2017lsh}. Consider the language modeling task where there are only a few words out of a large vocabulary in each sentence. There are several algorithms that impose sparsity on the Softmax layer to improve training time. However, getting around memory is still a major challenge. Since the distribution of words follows a power-law distribution, Sampled Softmax \citep{jean2014sampled} is commonly used to training language models. \cite{shrivastava2014asymmetric, vijayanarasimhan2014deep, lossICML2018} have proposed using approximate nearest-neighbor search to find the output classes that contain the highest gradients.

Our solution takes advantage of the sparsity present in the Embedding and Softmax layers, so the computational cost scales with the gradient sparsity. We directly insert the sparse gradients into the count-sketch, and then retrieve an approximation of the auxiliary variable. Furthermore, we can easily trade-off the capacity of the count-sketch to maintain the optimizer's performance, without increasing the cost of updating or querying the structure. In Section \ref{sec:theory}, we formally prove this graceful memory trade-off, by analyzing the convergence rate of our count-sketch optimizer.

On the 1-Billion Word dataset, we train an LSTM language model using the Adam optimizer, leveraging our count-sketch technique. By compressing the auxiliary variables for the Embedding and Softmax layers, we reduce the memory usage during training by \textbf{25$\%$} without any accuracy or performance penalty. For an Amazon extreme classification task with over \textbf{49.5} million classes, we reduce the training time by \textbf{38\%} by increasing the mini-batch size \textbf{3.5$\times$} using our count-sketch optimizer.

% Increase number of parameters in model to improve performance - Language Models and Neural Machine Translation
% Training large-scale models efficiently - Data Parallelism - Multi-GPU, Mixed Precision
% Limited memory capacity - Trade-off between more expressive model and faster training by increasing batch size
% Commonly used first-order optimizers require extra memory - 2-3x of the model's parameters
% Goal: Compress the auxiliary variables of the first-order optimizers
% Focus: Embedding and Softmax Layers - Extremely Sparse but Very Large
% Solution: Utilize the Count-Sketch data structure to compress the auxiliary variables while taking advantage of the sparsity naturally present in the Embedding and Softmax layers.

%\begin{enumerate}
%    \item Related to Random Projection Dimensionality Reduction
%    \item $\epsilon$-approximation with constant-time update and query operations
%    \item Compress updates into count-sketch, Recover approximate value with efficient query operation
%    \item Ideally suited for the ultra-high dimensional feature spaces and softmax layers with large output classes that are naturally sparse
%    \item We can easily control the size of the count-sketch to maintain accuracy. The low-rank approximation is limited to the rank-1 special case to have fast updates.
%    \item The count-sketch structure update cost scales with the sparsity level of the gradients. The low-rank approximation requires an expensive outer-product operation to reconstruct the entire matrix.
%\end{enumerate}

\section{Count-Sketch and Streaming Setting}
In the traditional streaming setting, we are given a high-dimensional vector ${\boldsymbol x} \in \mathbb{R}^d$ that is too costly to store in memory. We only see a very long sequence of updates over time. The only information available at time $t$ is of the form $(i, \Delta)$, which means that coordinate $i$ is updated by the amount $\Delta$. We are given a limited amount of storage, on the order of $\mathcal{O}(\log{d})$, which means that we can never store the entire vector. Sketching algorithms aim to estimate the value of current item $i$, after any number of updates using only $\mathcal{O}(\log{d})$ memory. %Accurate estimation of heavy coordinates is desirable.

The Count-Sketch is a popular algorithm for estimation in the streaming setting. Count-Sketch keeps a matrix of bins $\displaystyle \mS$ of size $v \times w \sim \mathcal{O}(\log{d})$, where $v$ and $w$ are chosen based on the desired accuracy guarantees. The algorithm uses $v$ random hash functions $h_j$ for $j \in \{1,\ 2,...,\ v\}$ to map the vector's components to $w$ different bins, $h_j:\{1,\ 2,..., \ d\} \rightarrow \{1, \ 2,...,\ w\}$. In particular, for any row $j$ of sketch $\displaystyle \mS$, component $i$ is hashed into bin $\displaystyle \mS_{j, h_j(i)}$. In addition, Count-Sketch uses $v$ random sign functions $s_j$ to map the components of the vectors randomly to $\{+1,\ -1\}$, $s_j:\{1,\ 2,..., \ d\} \rightarrow \{+1,-1\}$.

The Count-Sketch supports two operations: UPDATE(item $i$, increment $\Delta$) and QUERY(item $i$). The UPDATE operation updates the sketch with any observed increment. More formally, for an increment $\Delta$ to an item $i$, the sketch is updated by adding $s_j(i)\cdot\Delta$ to the cell $\mathcal{S}_{j,h_j(i)}, \forall  j \in \{1,\ 2,...,\ v\}$. The QUERY operation returns an estimate for component $i$, the median of all the $w$ different associated counters. If the updates are strictly non-negative, we return the minimum value across all the counters.

{\bf Count-Sketch Error:} 
\citep{charikar2002finding} Let $\hat{x}_i$ be the Count-Sketch estimate of component $i$ from vector $x$. For any component $x_i$, with probability $1-\delta$, a Count-Min Sketch matrix with width $\Theta(\frac{1}{\epsilon^2_1})$ and depth $\Theta(\log(\frac{d}{\delta}))$ satisfies:
$$x_i - \epsilon_1 \norm{x}_2 \le \hat{x}_i \le x_i + \epsilon_1 \norm{x}_2$$ 

{\bf Count-Min Sketch Error:} 
\citep{cormode2005improved} Let $\hat{x}_i$ be the Count-Min Sketch estimate of component $i$ from vector $x$. For any component $x_i$, with probability $1-\delta$, a Count-Min Sketch matrix with width $\Theta(\frac{1}{\epsilon_1})$ and depth $\Theta(\log(\frac{d}{\delta}))$ satisfies:
$$x_i \le \hat{x}_i \le x_i + \epsilon_1 \norm{x}_1$$
 
\begin{algorithm}[ht]
\caption{Count-Sketch Tensor}
\label{alg:CS}
\begin{algorithmic}
\STATE $v$ universal hash functions $h_j$
\STATE $v$ random sign functions $s_j$
\STATE  Initialize count-sketch tensor $\displaystyle \tS \in \mathbb{R}^{v,w,d}=0$

\hrulefill

\STATE  {\bf UPDATE(Count-Sketch $\displaystyle \tS$, item i, update $\Delta \in \mathbb{R}^d$)}: 
\STATE Update component $i$ with update $\Delta$
\FOR{$j=1$ {\bfseries to} $v$}
\STATE $\displaystyle \tS_{j,h_j(i),:} \gets  \displaystyle \tS_{j,h_j(i),:} + s_j(i) \cdot \Delta$
\ENDFOR

\hrulefill

\STATE {\bf QUERY(Count-Sketch $\displaystyle \tS$, item i, Function $F$):} 
\STATE Query sketch for an estimate for item $i$
\STATE $F$ - MIN for non-negative values; otherwise MEDIAN
\STATE $F_{j \in\{1,2,...,v\}} (s_j(i) \cdot \displaystyle \tS_{j,h_j(i),:})$
\end{algorithmic}
\end{algorithm}

\section{Intuition}
Our goal is to compress the auxiliary variables without incurring significant accuracy loss. Unfortunately, selecting the appropriate compression scheme is not clear without any additional information on the parameter distribution. The challenge is that the parameter distribution can change over time, so any static assumption on the approximation is likely to hurt accuracy. Fortunately, in this section we show that there is a potential solution. 

\textbf{Power Law in Auxiliary Variables over Time:} In Figure~\ref{fig:pl}, we plot the auxiliary variables sorted according to their absolute values during training. To understand the dynamics over time, we show the parameters at two different epochs 5 and 40. The plots clearly indicate a power law behavior where only a few parameters have large magnitudes. In Figure~\ref{fig:pl_50thd}, we confirm this behavior for every iteration by plotting the midpoint dividing the head and tails. The auxiliary variables have long tails throughout the training process. Also, this behavior is invariant across the two datasets - (Wikitext-2 and Image-Net). To the best of our knowledge, this is the first work that empirically shows the existence of a power law distribution behavior in the gradients and auxiliary variables while training. To dig deeper, we also show the identities of top-100 parameters (the head of power law distribution) for epochs 5, 20, and 40 in Figure~\ref{fig:pl}. The top identities change over time, which makes it difficult to cluster parameters into predefined, static clusters. 

\textbf{Power law and linear sequence of updates:} In summary, we need to compress a power law distribution where the top-k identities are constantly changing. Fortunately, the auxiliary variables are updated in a linear fashion. The updates can be written as a linear operator over updates (See Section~\ref{sec:algorithm}). The count-sketch is a dynamic, low-memory data structure, which preserves high magnitude parameters accurately, while allowing for any sequence of linear updates. The linearity of updates allows us to guarantee that the count-sketch provides an accurate estimation of parameters with high probability at every stage in the iteration. The power law distribution and linear updates make sketching-based ideas a perfect fit for this problem. 

\begin{figure*} [ht]
\begin{center}
\mbox{
\includegraphics[width=0.33\textwidth]{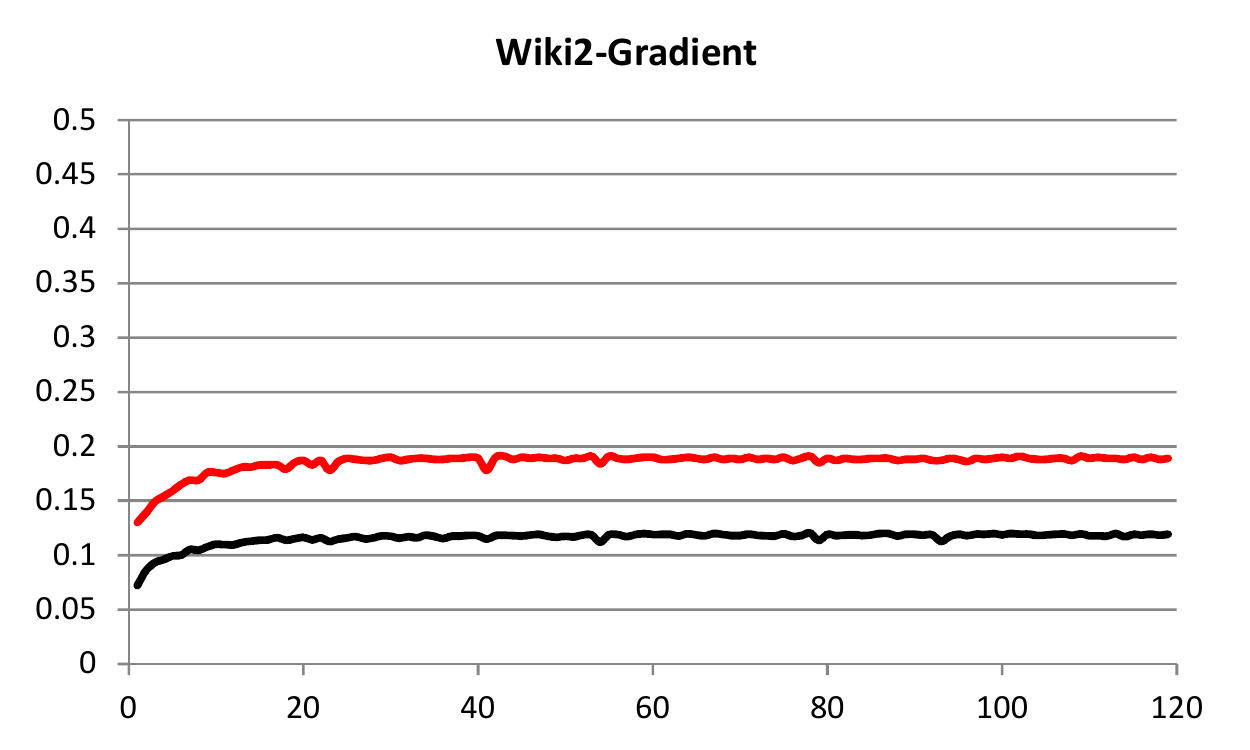}
\includegraphics[width=0.33\textwidth]{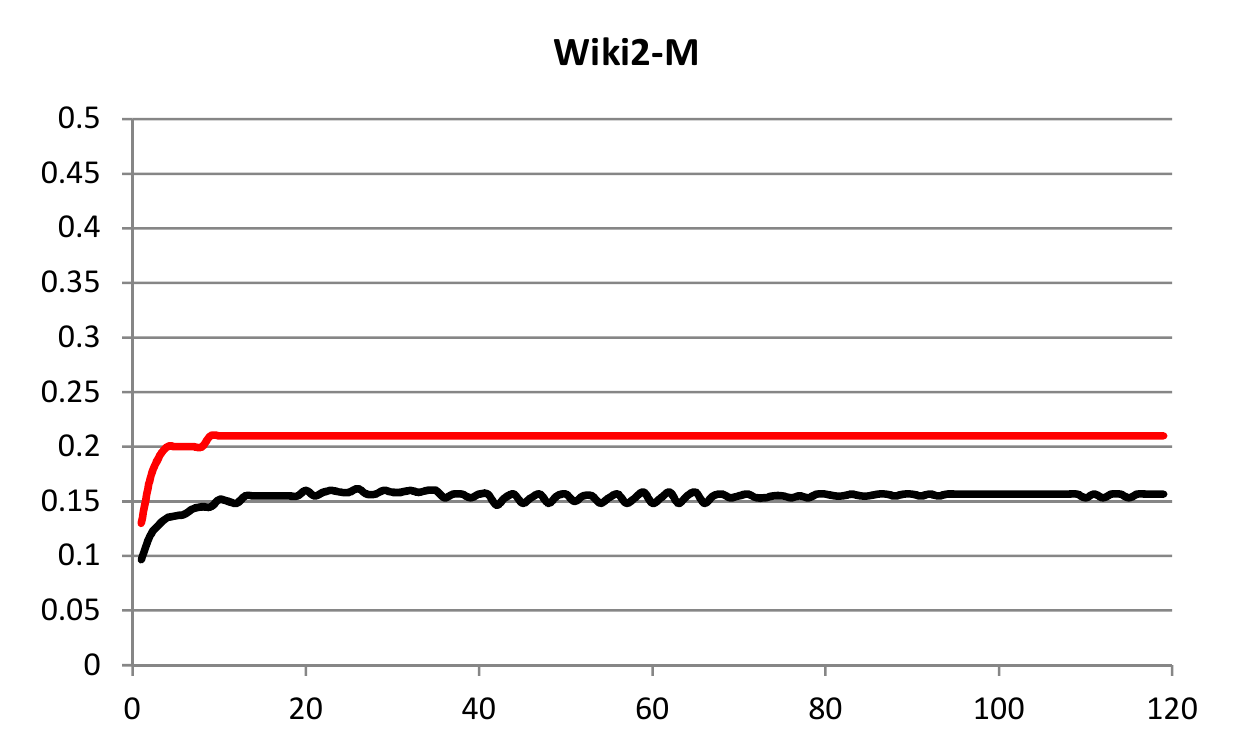}
\includegraphics[width=0.33\textwidth]{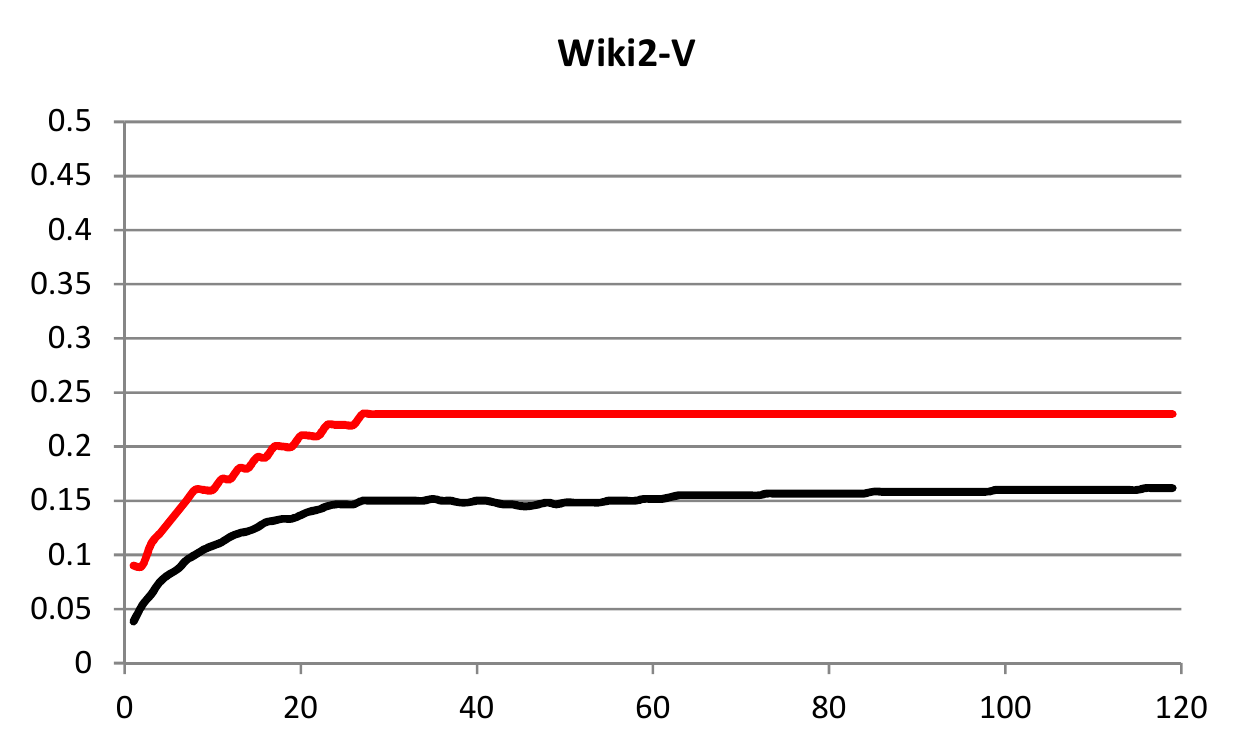}
}
\\
\mbox{
\includegraphics[width=0.33\textwidth]{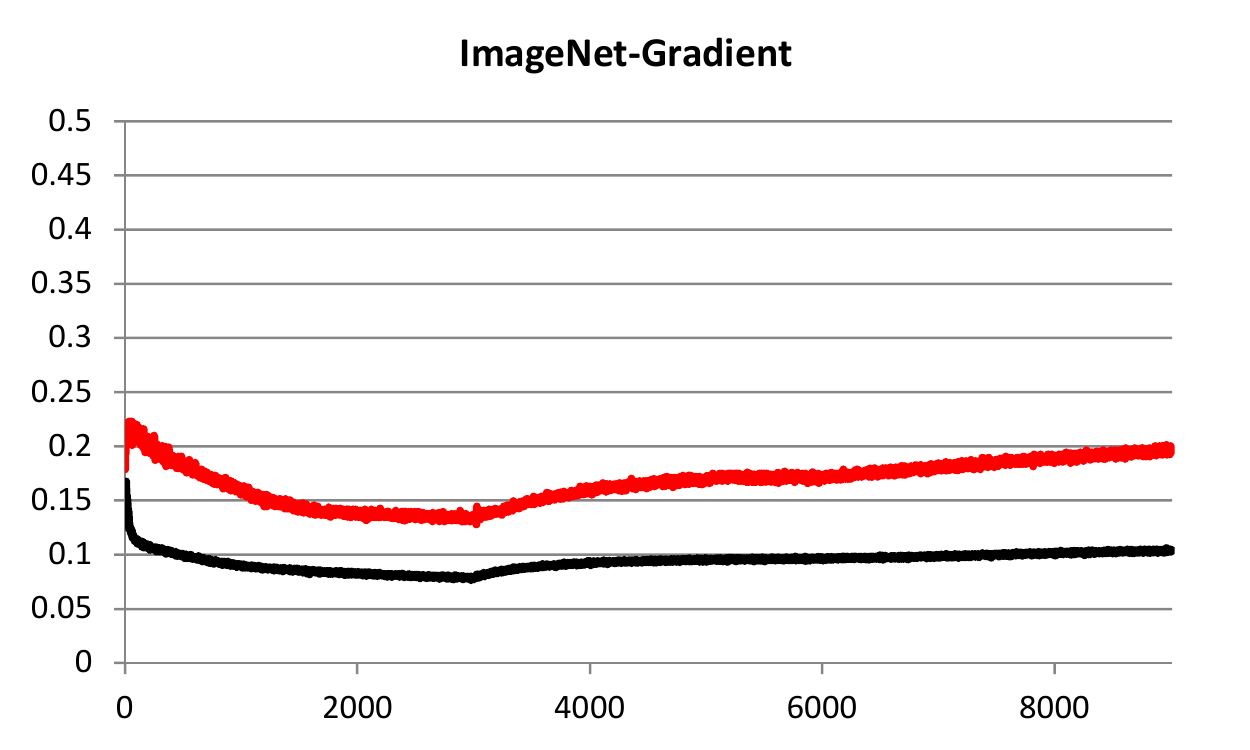}
\includegraphics[width=0.33\textwidth]{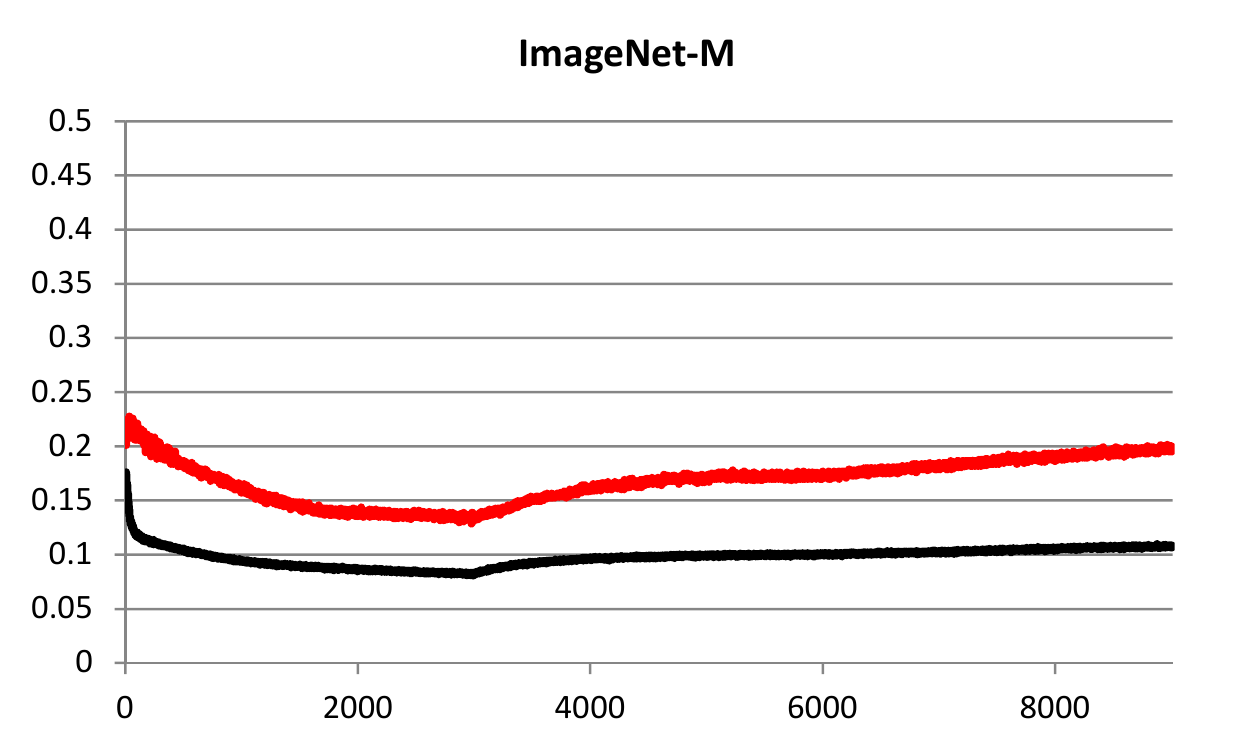}
\includegraphics[width=0.33\textwidth]{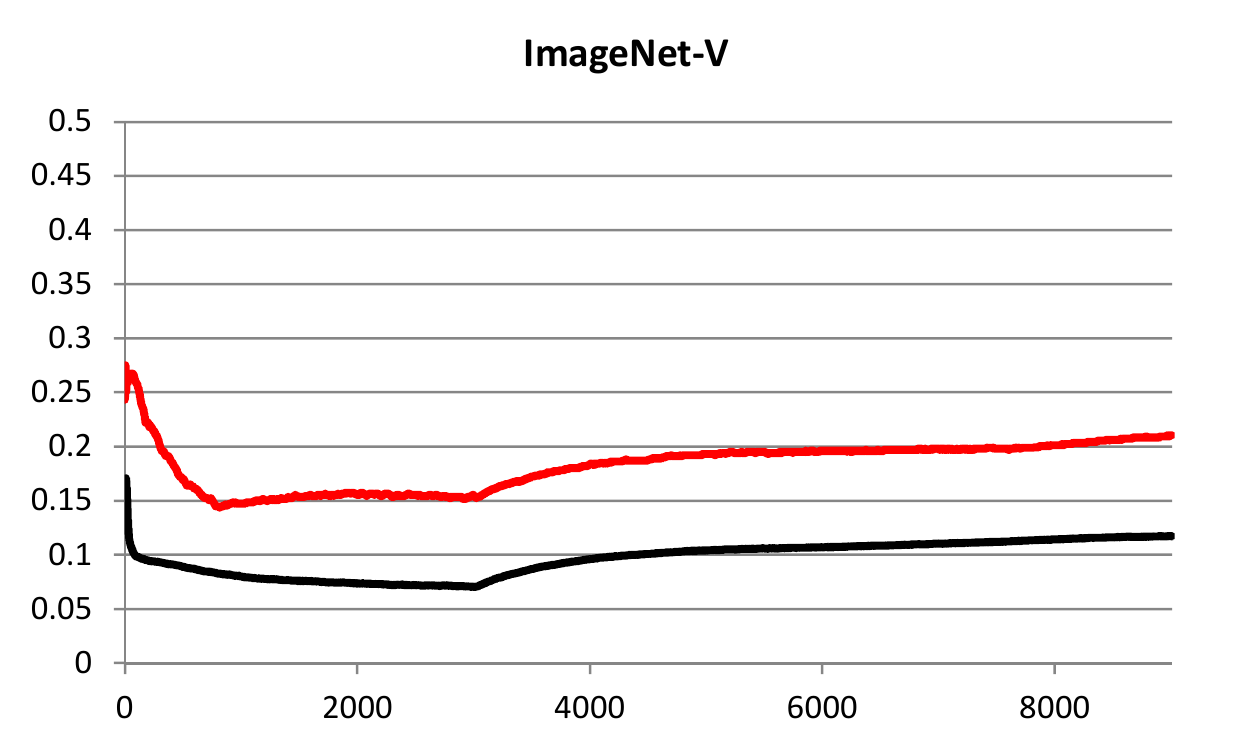}
}
\end{center}
\caption{An empirical demonstration showing that the model's gradients and the optimizer's auxiliary variables follow a power-law distribution. The count-sketch data structure approximates the heavy hitter entries with greater accuracy. Therefore, this experiment implies that the count-sketch data structure is appropriate for compressing the auxiliary variables. The X-axis is the number of iterations during training time. The Y-axis is the 50\% threshold that marks the midpoint dividing the head and the tail of the distribution. For a uniform distribution, the midpoint is at 0.5. However, the 50\% threshold for the gradients and auxiliary variables is less than 0.2 on average, indicating that they follow a power law distribution. The red line marks the maximum threshold for all layers, while the black line represents the average threshold.}
\label{fig:pl_50thd}
\end{figure*}

\begin{figure*} [ht]
\begin{center}
\mbox{
\includegraphics[width=0.2425\textwidth]{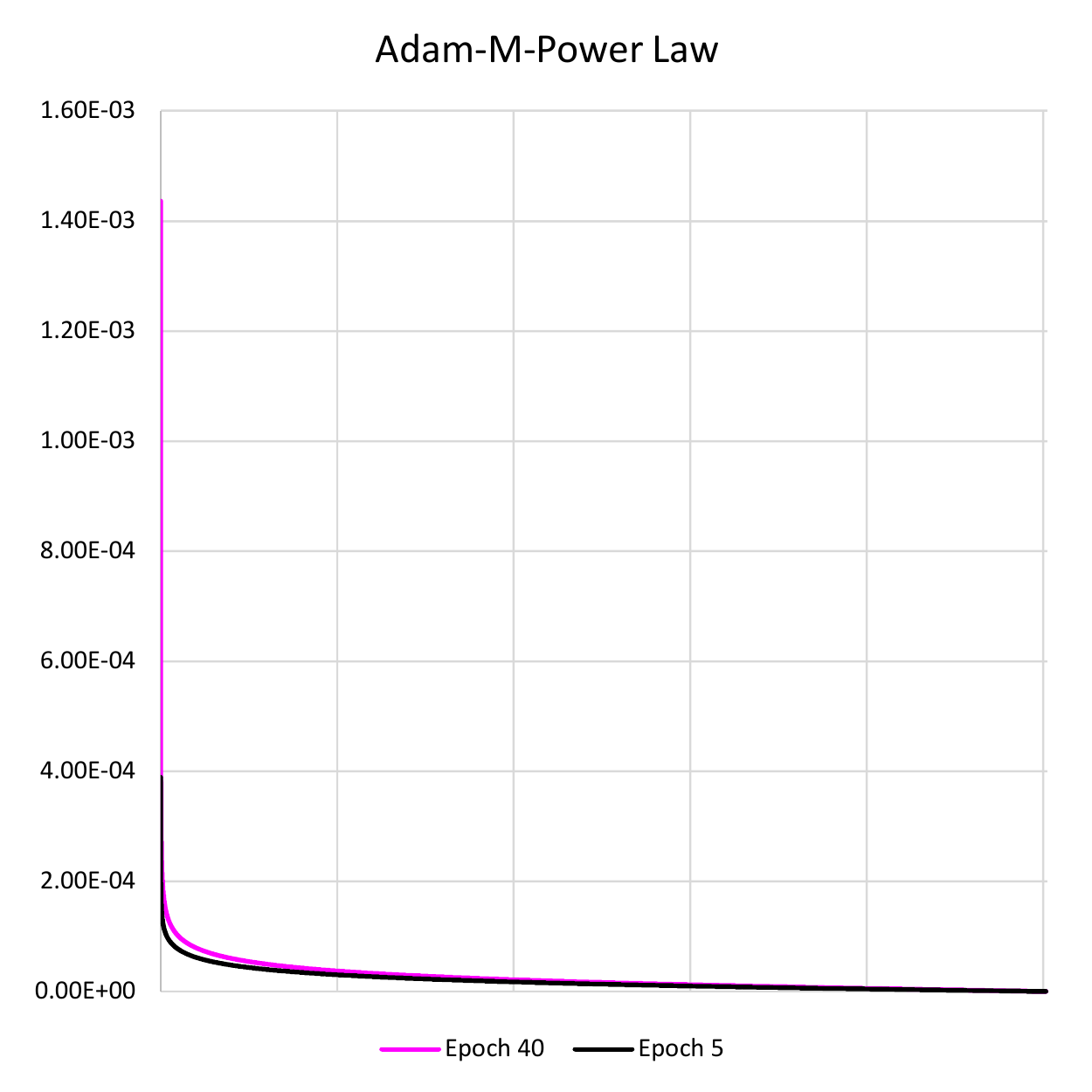}
\includegraphics[width=0.2425\textwidth]{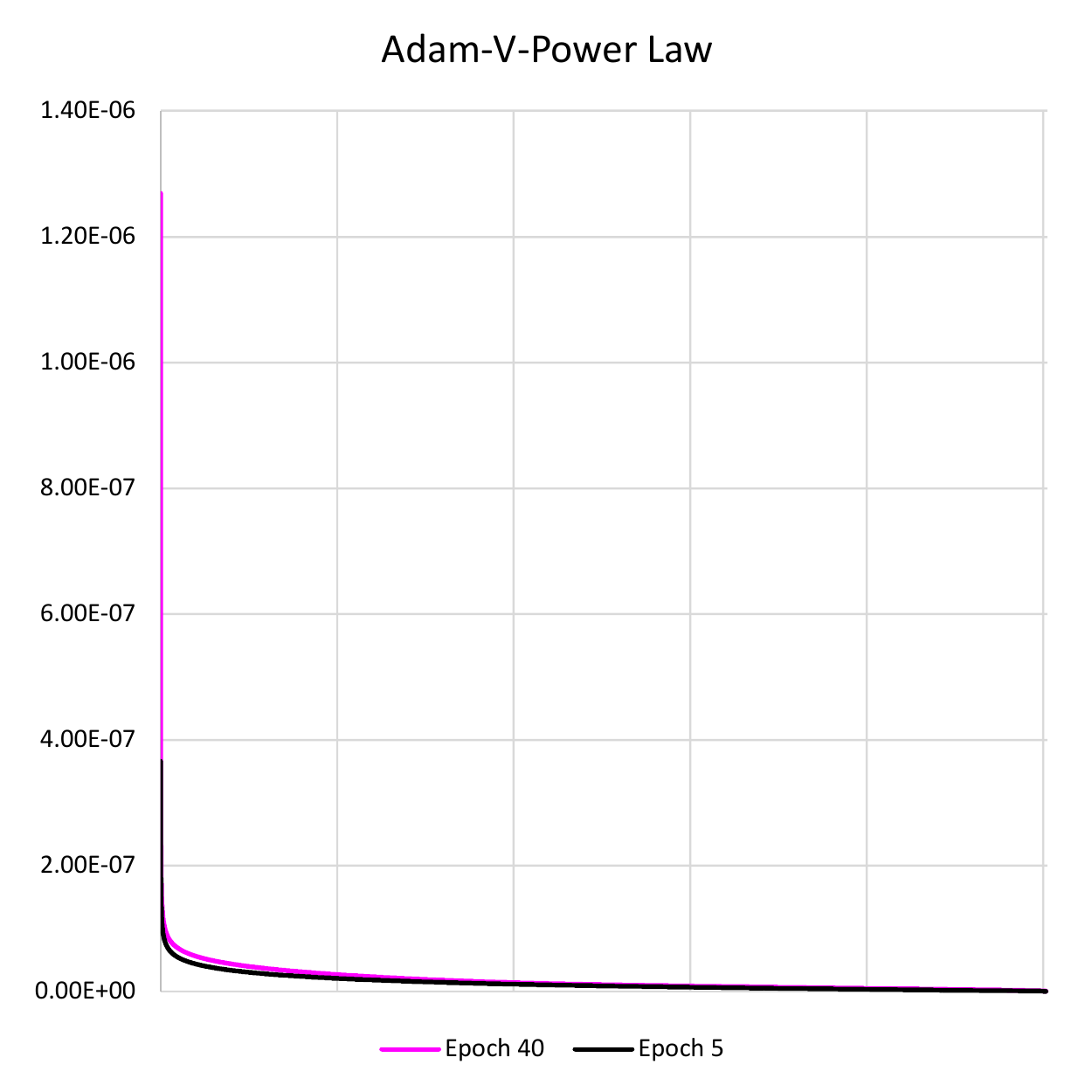}
\includegraphics[width=0.2425\textwidth]{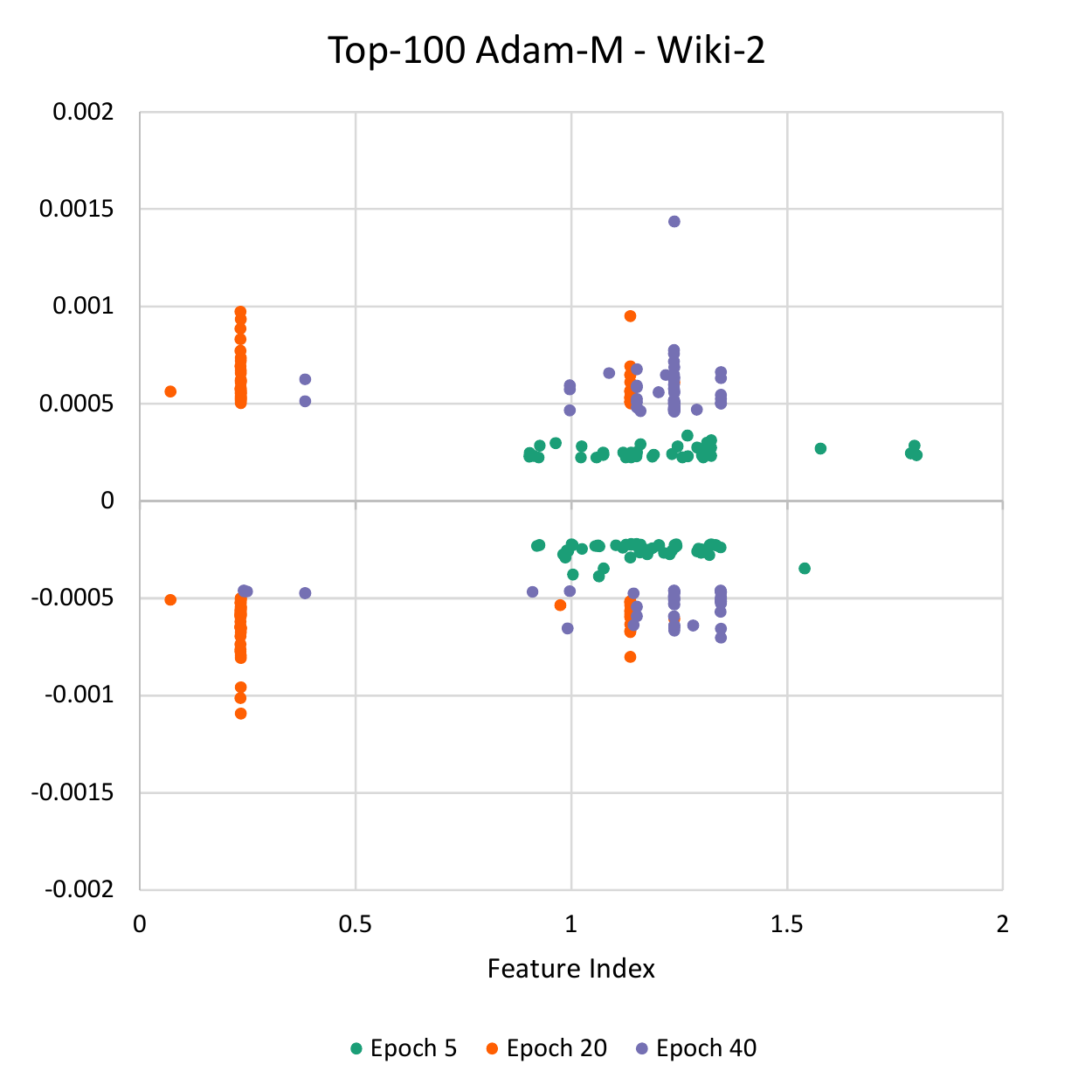}
\includegraphics[width=0.2425\textwidth]{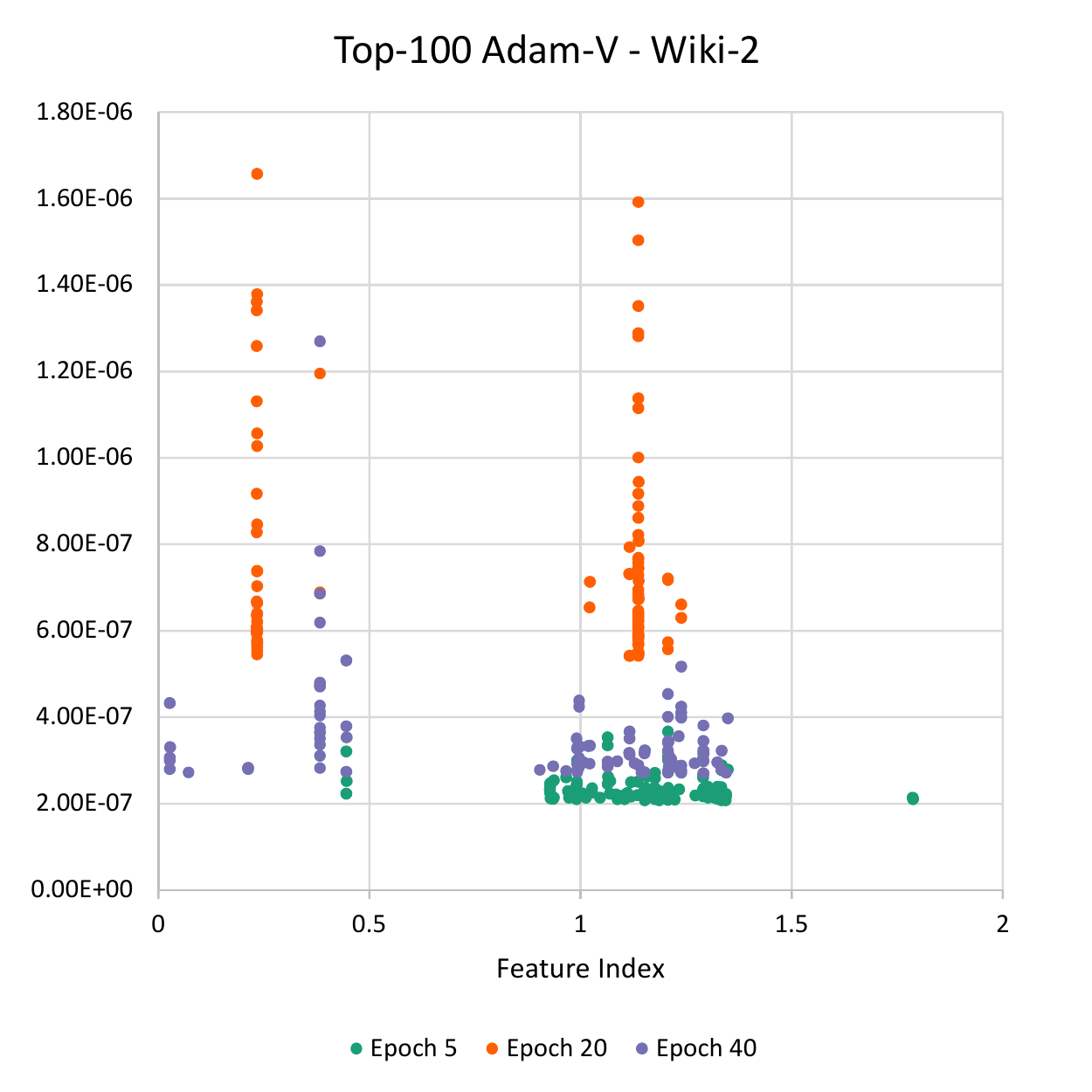}
}
\end{center}
\caption{The optimizer's auxiliary variables follow a power-law distribution, but the features associated with top-k values change during training. The X-Axis is the feature ID, while the Y-Axis is the magnitude. The first two charts show the sorted absolute values for the auxiliary variables at different training epochs. The last two charts plot the top 100 features and their magnitudes. We plot the 1st and 2nd moments of the Adam Optimizer for an LSTM weight matrix trained on the Wiki2 dataset.}
\label{fig:pl}
\end{figure*}

\section{Count-Sketch Optimizers}
\label{sec:algorithm}
% Motivation
A major chunk of the parameters in the deep network are contained in the fully-connected layers \citep{han2015deep}. Fortunately, for the embedding and softmax layers, the set of active features or classes and their corresponding gradient updates are sparse. Our insight is to use the count-sketch data structure to accurately represent the auxiliary variables in a compressed manner. We will insert the sparse gradient information into the count-sketch and retrieve an approximate value for the auxiliary variable whenever needed.

% Count-Sketch Tensor
In the deep learning setting, the high-dimensional vector ${\boldsymbol \beta}$ is analogous to the matrices used to represent the auxiliary variables. The auxiliary variables are represented with $\mathbb{R}^{n,d}$ matrices where $n$ is the number of features in the embedding layer or the number of classes in the softmax layer. Since the dimensionality of the columns $v$ is usually in the low thousands ($<10K$), we represent the auxiliary variables with a count-sketch tensor $\mathbb{R}^{v,w,d}$ where $v \cdot w \ll n$. This count-sketch tensor preserves structured sparsity where values are read from memory in contiguous chunks along the last dimension of the tensor. See Fig.~\ref{fig:cs_tensor} for a visualization. This tensor structure maintains high performance with GPUs and CPU SIMD vector instructions. On the other hand, the $n$ rows are compressed by randomly combining features and classes together.

\begin{figure}
\begin{center}
\includegraphics[width=0.4\textwidth]{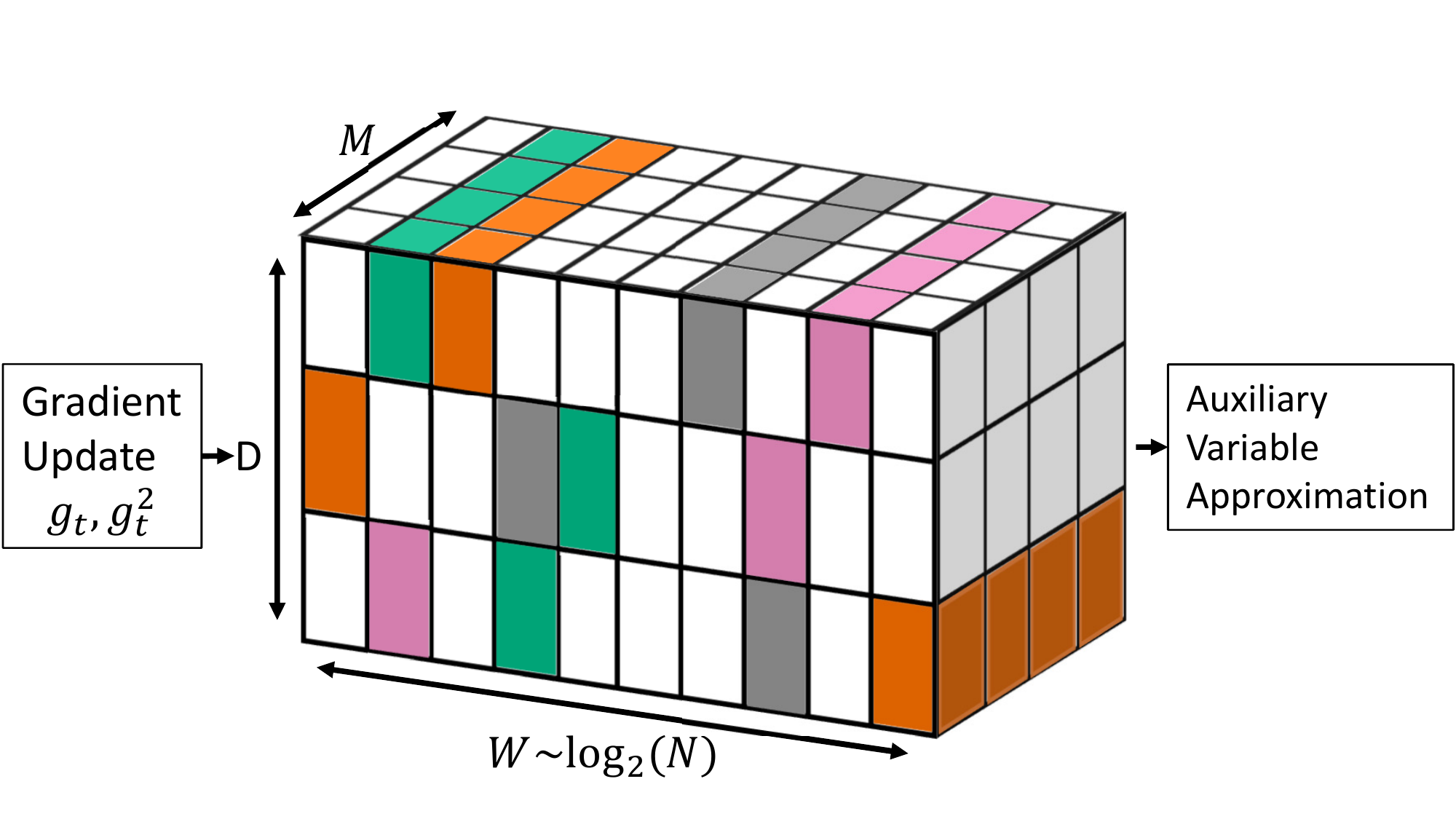}
\end{center}
\caption{Visualization of Count Sketch Tensor. Each color represents a unique feature. For each row, each feature is mapped randomly to a different vector. Each vector is read from and written to memory in contiguous chunks. Preserving the last dimension of the auxiliary variable keeps structure sparsity in the count-sketch\ data structure, which is necessary for high performance.
}
\vspace{-0.1in}
\label{fig:cs_tensor}
\end{figure}

% Overview of Momentum, Adagrad, Adam
Here is a brief overview of three popular first-order optimizers whose auxiliary variables we seek to compress: Momentum \citep{sutskever2013mtm, polyak1964mtm} remembers a history of gradient updates, which smooths out random oscillations and accelerates convergence. Adaptive gradient descent algorithms alter the learning rate for each feature based on the frequency of its updates. Sparse, rare features are given larger updates and a higher learning rates. These methods track a history of squared gradients for each feature. Adagrad \citep{duchi2011adagrad} divides the gradient by the square root of the cumulative squared gradient. Adam \citep{kingma2014adam} combines momentum and adaptive learning rates together, so it tracks an exponential average of the gradients and squared gradients.

% Update-Form of Optimizer
% Replace += with Count-Sketch Update-Query operation
% Advantage of Count-Sketch
The count-sketch data structure expects to receive a stream of updates $\Delta$. For the Momentum and Adam optimizers, we need to transform the update operation into a form that is compatible with the count-sketch. For an auxiliary variable $\displaystyle \mX$, the desired update operation is $\displaystyle \mX \pluseq \Delta$. Given the appropriate update operation, we replace the addition assignment operator $\pluseq$ for the original matrix with the Update-Query operation for the Count-Sketch Tensor.

For Momentum, the update rule, given some gradient $g_t$, is $m_t = \gamma \cdot m_{t-1} + g_t \leftrightarrow m_t \pluseq (\gamma - 1) \cdot m_{t-1} + g_t$. For the Adam optimizer, given some constant $c$ and an update $\Delta$, the update rule for the exponential moving average is $x_t = c \cdot x_{t-1} + (1 - c) \cdot \Delta \leftrightarrow x_t \pluseq (1 - c) \cdot (\Delta - x_{t-1})$.

%\begin{multline*}
%m_t = \gamma \cdot m_{t-1} + g_t \leftrightarrow m_t \pluseq (\gamma - 1) \cdot m_{t-1} + g_t   
%\end{multline*}

%\begin{multline*}
%x_t = c \cdot x_{t-1} + (1 - c) \cdot \delta \leftrightarrow x_t \pluseq (1 - c) \cdot (\delta - x_{t-1})
%\end{multline*}

The Count-Sketch is essentially a plug and play replacement that saves memory, while retaining the speed and accuracy of the original matrix. Normally, algorithms that compress memory to save space are slower than their dense counterparts. However, the count-sketch can leverage sparsity by lazily performing updates with high efficiency. In addition, we can gracefully increase the size of the count-sketch for greater accuracy with minimal additional computational cost.

\begin{algorithm}
\caption{Momentum - Count Sketch Optimizer}
\label{alg:mtm_alg}
\begin{algorithmic}
\STATE  Initialize Count-Sketch Tensor $\displaystyle \tM \in \mathbb{R}^{v,w,d}=0$
\STATE $v$ universal hash functions $h_j$
\STATE $v$ random sign functions $s_j$
\STATE Decay Rate $\gamma$, Learning Rate $\eta$

\hrulefill
\STATE  {\bf MOMENTUM}
\STATE  {\bf (Item $i$, Parameter $x \in \mathbb{R}^d$, Gradient $g_t \in \mathbb{R}^d$)}:
\STATE $m_{t-1} \gets$ Query($\displaystyle \tM$, $i$, MEDIAN$)$
\STATE $\Delta_M \gets (\gamma - 1) \cdot m_{t-1} + g_t$
\STATE Update$(\displaystyle \tM$, $i$, $\Delta_M)$
\STATE $\hat{m}_{t} \gets$ Query($\displaystyle \tM$, $i$, MEDIAN$)$
\STATE $x_t = x_{t-1} - \eta_t \cdot m_t$
\end{algorithmic}
\end{algorithm}

\begin{algorithm}
\caption{Adagrad - Count Sketch Optimizer}
\label{alg:adagrad_alg}
\begin{algorithmic}
\STATE  Initialize Count-Min Sketch Tensor $\displaystyle \tV \in \mathbb{R}^{v,w,d}=0$
\STATE $v$ universal hash functions $h_j$
\STATE Learning Rate $\eta$

\hrulefill
\STATE  {\bf ADAGRAD}
\STATE  {\bf (Item $i$, Parameter $x \in \mathbb{R}^d$, Gradient $g_t \in \mathbb{R}^d$)}:
\STATE $\Delta_V \gets g_t^2$
\STATE UPDATE($\displaystyle \tV$, $i$, $\Delta_V)$
\STATE $v_{t} \gets$ QUERY($\displaystyle \tV$, $i$, MIN$)$
\STATE $x_t = x_{t-1} - \eta_t \cdot \frac{g_t}{\sqrt{v_t} + \epsilon}$
\end{algorithmic}
\end{algorithm}

\begin{algorithm}
\caption{Adam - Count Sketch Optimizer}
\label{alg:adam_alg}
\begin{algorithmic}
\STATE Initialize Count-Sketch Tensor $\displaystyle \tM \in \mathbb{R}^{v,w,d}=0$
\STATE Initialize Count-Min-Sketch Tensor $\displaystyle \tV \in \mathbb{R}^{v,w,d}=0$
\STATE $v$ universal hash functions $h_j$
\STATE $v$ random sign functions $s_j$
\STATE 1st Moment Decay Rate $\beta_1$, 2nd Moment Decay Rate $\beta_2$
\STATE Learning Rate $\eta$

\hrulefill
\STATE  {\bf ADAM}
\STATE  {\bf (Item $i$, Parameter $x \in \mathbb{R}^d$, Gradient $g_t \in \mathbb{R}^d$)}:
\STATE // Count-Sketch - 1st Moment
\STATE $m_{t-1} \gets$ Query($\displaystyle \tM$, $i$, MEDIAN$)$
\STATE $\Delta_M \gets (1 - \beta_1)(g_t - m_{t-1})$
\STATE Update$(\displaystyle \tM$, $i$, $\Delta_M)$
\STATE $m_{t} \gets$ Query($\displaystyle \tM$, $i$, MEDIAN$)$

\hrulefill
\STATE // Count-Min Sketch - 2nd Moment
\STATE $v_{t-1} \gets$ Query($\displaystyle \tM$, $i$, MIN$)$
\STATE $\Delta_V \gets (1 - \beta_2)(g_t^2 - v_{t-1})$
\STATE Update$(\displaystyle \tV$, $i$, $\Delta_V)$
\STATE $v_{t} \gets$ Query($\displaystyle \tM$, $i$, MIN$)$

\hrulefill
\STATE $\hat{m}_t = m_t/(1 - \beta_1^t)$
\STATE $\hat{v}_t = v_t/(1 - \beta_2^t)$
\STATE $x_t = x_{t-1} - \eta_t \cdot \frac{\hat{m_t}}{\sqrt{\hat{v_t}} + \epsilon}$
\end{algorithmic}
\end{algorithm}

\textbf{Count-Min Sketch Cleaning Heuristic:} Since the Count-Min Sketch only accepts non-negative values, it always overestimates the desired value. The Count-Min Sketch is used to estimate the adaptive learning rate for the Adagrad and Adam optimizers. Therefore, an overestimate will prematurely slow the learning rate for certain elements. Our heuristic solution is to clean the sketch periodically by multiplying the tensor by a constant $\alpha$ where $0 \le \alpha \le 1$ every $C$ iterations. Instead of this heuristic, an alternative is to use principled adaptive sketches~\cite{shrivastava2016time}, which can continuously clean the sketch and decay the overestimates over time. 

Periodic cleaning works well with the Count-Min Sketch because it provides a better estimate for the top-$k$ elements. During training, the accumulation of updates allows for the heavy hitter estimates to emerge in the sketch \cite{missionICML2018}. Due to stochastic gradient descent, there is a certain amount of noise in the gradient, so cleaning immediately after each update destroys the internal state of the sketch. Furthermore, cleaning reduces the scale of the sketch, reducing the overall noise level. If the signal to noise ratio is too high, future heavy hitter are ignored because there values are equal to the noise in the sketch. 

% Due to stochastic gradient descent, there is a certain amount of noise in the gradients.
% A property of the Count-Sketch data structure is that it preserves the heavy hitters in the stream with greater accuracy.

% When the scale / energy of the sketch is high, then learning rate slows prematurely for certain elements. 
% Periodic cleaning allows the heavy hitter estimates to emerge in the sketch while reducing the overall scale of the sketch estimates.
% In addition, the signal to noise ratio is high, crowding out future important values.

\section {Theoretical Analysis} \label{sec:theory}
For stochastic non-convex optimization \citep{yogi_nips_2018}, we measure how the algorithm converges to a stationary point at iteration $x_t$---i.e., $\norm{\nabla f(x_t)}^2 \le c$ for some small constant $c$.  In our analysis, we focus on the Count-Min Sketch Adam optimizer where we do not track the 1st moment---i.e., $\beta_1=0$. This optimizer was used in the Amazon Extreme Classification task (See Section \ref{sec:amz}) in order to save additional memory, similar to the Adafactor optimizer \cite{adafactorICML2018}. 

We assume that the function $f$ is $L$-smooth with bounded gradients:  Function $f$ has bounded gradients - $[f(x_t)]_i \le G_i,\forall x \in \mathbb{R}^d, i \in [d], G = \norm{\vec{G}}_\infty$. In addition, we receive an unbiased stochastic gradient estimate $g_t$ with fixed variance $\sigma^2$. Then, the following theorem holds:

%\begin{enumerate}
%    \item Function $f$ is L-Smooth: There exists a constant $L$ such that $\norm{\nabla f(x) - \nabla f(y)} \le L\norm{x - y},  \forall x,y \in \mathbb{R}^d$
%    \item Function $f$ has bounded gradients: \\
%    $[f(x_t)]_i \le G,~\forall x \in \mathbb{R}^d, i \in [d]$
%    \item The stochastic gradient oracle provides us with an unbiased estimate with fixed variance - $g_{t,i} = [\nabla f(x_t, \xi_t)]_i$ where $\xi_t$ represents the randomness at iteration $t$ because of mini-batch sampling. $\mathbb{E} [g_{t,i}] = [\nabla f(x_t)]_i$ $\mathbb{E} [(g_{t,i} - [\nabla f(x_t)]_i)^2] \le \sigma_i$
%\end{enumerate}

\begin{theorem} \label{eq:cs_thm}
Let the learning rate $\eta_t = \eta, \forall t \in [T]$. Assume $\beta_2$, $\eta$, and $\epsilon$ are selected such that $\eta \le \frac{\epsilon}{2L}$ and $\sqrt{1-\beta_2} \le \frac{\epsilon}{4G}$. Given a Count-Min Sketch matrix with width $\Theta(\frac{1}{\epsilon_1})$ and depth $\Theta(\log(\frac{dT}{\delta}))$, we have the following bound that holds for Count-Min Sketch Adam with probability $(1-\delta)$ where $M = (\frac{G \sqrt{1-\beta_2}}{\epsilon^2} \sum^{d}_{i=0} \norm{\vec{G}}^2_2)$:
\begin{equation*}
\min_{t}~\mathbb{E} \norm{\nabla f(x_t)}^2 \le  \mathcal{O} \Big( \frac{f(x_0) - f(x_*)}{\eta T} + \sigma^2 + \epsilon_1 M \Big) 
\end{equation*}
\end{theorem}

The proof of Theorem \ref{eq:cs_thm} is found in the Appendix. For comparison, we have the convergence bound from \citep{yogi_nips_2018} for the standard Adam optimizer where $\beta_1=0$:

\begin{equation*}
\label{eq:adam_thm}
\min_{t}~\mathbb{E} \norm{\nabla f(x_t)}^2 \le \mathcal{O} \Big( \frac{f(x_0) - f(x_*)}{\eta T} + \sigma^2 \Big) 
\end{equation*}

\textbf{Discussion:}  The bounds are similar except for the additional term caused by the Count-Min Sketch approximation. The theorem states that the Count-Min Sketch Adam converges to a region around a stationary point with radius $O(\sigma^2 + \epsilon_1 M)$. The additional error term $\epsilon_1 M$ depends on the adaptivity of the optimizer $\beta_2$, the error rate $\epsilon_1$ of the sketch, and the gradient norm $\norm{G}^2_2$. The error rate $\epsilon_1$ is proportional to the width of the sketch $\epsilon_1 = 1/w$ and corresponds with the number of collisions along each row in the sketch. We can improve convergence gracefully by increasing the sketch's width, which reduces the error caused when multiple components collide in the same bin. In practice, we bound the gradient norm to reasonable constant to prevent instability---i.e., $\norm{G}^2_2 \le C^2$. When the sketch width $w=\Theta(d)$, the error term becomes a small constant. 

Note that the gradient norm decreases over time. Thus, the error caused by the count-sketch approximation decreases as the algorithm progresses, and we can shrink the sketch. A nice property of the count-sketch data structure is that you can add one half of the sketch to the other, reducing its size by half while maintaining its accuracy guarantees. Please see \cite{matusevych2012hokusai} for more details. 

The failure probability $\delta$ of exceeding the Count-Min Sketch error bound is proportional to the depth of the sketch $\delta = dT/e^v$. In our theoretical results, the depth of the sketch depends logarithmically on the number of parameters $d$ and the number of time steps $T$. However, our experiments show that a modest depth size of 3-5 is sufficient.

\section{Related Work}
{\bf Feature Compression:}  A straight-forward option is to use dimensionality reduction techniques to minimize the number of features, which in turn decreases the size of the model and optimizer simultaneously. \cite{hashNIPS2017} describes a hash embedding scheme where the output embedding for a feature is a weighted sum between the embedding vectors and the weight vector. Their goal was to minimize the size of the embedding layer while preserving its flexibility to model large vocabularies. However, dramatically reducing the feature space may sacrifice model accuracy. For example, training the BERT language model \cite{devlin2018bert} on a GPU with 12-16 GB memory requires a smaller, less effective architecture than the full-sized model trained on the 64 GB Google TPU.

{\bf Gradient Checkpointing:} \cite{siskind2018divide, chen2016training} describe an orthogonal approach where training an $N$-layer neural network requires $\sqrt{N}$ memory. Their insight was that storing the activations for the back-propagation pass is the most memory-intensive part of training. Instead of storing all the activations, their algorithm checkpoints certain sections of the neural network and lazily recomputes the activations during the back-propagation phase. In other words, their approach saves memory by sacrificing extra computation time.

{\bf Low-Rank Approximation:} A low-rank approximation has the potential to reduce the number of parameters from $O(nd)$ to $O(nr + rd)$ where $r \ll \min(n,d)$. However, updating the low-rank matrices is non-trivial. \cite{adafactorICML2018} demonstrated that there exists a unique, fast update rule for a rank-1 approximation that minimizes the I-divergence between the approximation and original matrix. Their rank-1 approximation was limited to non-negative matrices, so only the second moment of the Adam optimizer was compressed in their experiments. The drawback of this approach is that it requires materializing the entire matrix via an outer-product, which is prohibitive for large-scale embedding and softmax layers. In addition, since their update rule only applies for rank-1 vectors, their approach lacks the flexibility to increase the model's memory capacity gracefully.

{\bf Count-Sketch:} The original objective of the Count-Sketch data structure was to estimate the frequency of various events in the streaming setting. Recently, \cite{missionICML2018, sketchSIGMOD2018} demonstrated that the Count-Sketch can learn a compressed model that accurately preserves the features with the largest weights. Their objective focused on feature extraction in ultra-high dimensional settings and was limited to simple, linear models. In this work, we seek to use the Count-Sketch to preserve the different auxiliary variables maintained by commonly used first-order optimizers. The ideal solution is for the memory cost of the optimizer to grow sub-linearly with the model size, giving us the flexibility to increase the model's capacity.

\begin{table}
\caption{Trade-offs between the Count-Sketch and Low-Rank Approximation. $k$ is the number of active features or classes. $r$ is the rank of the two factors where $r << min(n, d)$. The Count-Sketch data structure is ideally suited for the sparse embedding and softmax layers because it does not require a matrix-multiplication to reconstruct the entire auxiliary variable. }
\begin{center}
    \begin{tabular}{ |r|r|r| } 
    \hline
     Type & Count-Sketch &  Low-Rank \\
    \hline
    Memory & $\mathcal{O}(n \cdot \log{d})$ & $\mathcal{O}(nr+rd)$ \\
    Gradient Type & Sparse & Dense\\
    Memory Control & Flexible & Fixed \\
    Query Time & $\mathcal{O}(nk)$ & $\mathcal{O}(nrm)$ \\
    \hline
    \end{tabular}
\end{center}
\label{fig:tradeoffs}
\end{table}

\section{Experiments}
All of the experiments were performed with the PyTorch framework on a single machine - 2x Intel Xeon E5-2660 v4 processors (28 cores / 56 threads) with 512 GB of memory using a single Nvidia Tesla V100. The code\footnote{\url{https://github.com/rdspring1/Count-Sketch-Optimizers}} for the Count-Sketch Optimizer is available online. We designed the experiments to answer these questions:
\begin{enumerate}
\item Does the model's gradients and the optimizer's auxiliary variables follow a power-law distribution?
\item How accurate is our estimate of the auxiliary variables retrieved from the count-sketch data structure?
\item What the effect of cleaning the count-min sketch on convergence time and accuracy?
\item How well does our count-sketch optimizer compare against the low-rank approximation given the same number of parameters?
\item Does our count-sketch optimizer match original baseline in terms of speed and accuracy?
\end{enumerate}
Here are the five datasets used in the experiments:
\begin{enumerate}
    \item Wikitext-2 \citep{merity2016pointer} - This dataset was extracted from Wikipedia and contains 2M training tokens with a vocabulary size of 33,278. (10.8 MB)
    \item Wikitext-103 \citep{merity2016pointer} - A larger version of the Wikitext-2 dataset that contains 103M training tokens and its vocabulary size is 267,735. (539.2 MB)
    \item 1-Billion Word (LM1B) \citep{chelba2013gbw} - This large-scale corpus contains 0.8 billion training tokens and a vocabulary with 793,471 words. (4.1 GB) An open-sourced PyTorch model is available online \footnote{\url{https://github.com/rdspring1/PyTorch_GBW_LM}}
    \item MegaFace - A facial recognition dataset derived from MegaFace (Challenge 2) \footnote{\url{http://megaface.cs.washington.edu/}}. Each person is a candidate class, but we only select classes with at least 10 images. Thus, this sampled dataset contains 1,943,802 examples with 80,204 classes. 10K images are randomly sampled to create the test dataset. (4 GB)
    \item Amazon - This sampled recommendation dataset contains 70.3 million examples and over 49.5 million object classes. (20.9 GB)
\end{enumerate}
We implemented the following approaches to compare and contrast against our approach: 
\begin{enumerate}
    \item Non-Negative Matrix Factorization (NMF) Rank-1 --- This decomposition minimizes the I-divergence between the auxiliary variable and the approximation formed from two rank-1 vectors. However, it is limited to non-negative matrices, so it cannot compress the auxiliary variables for Momentum or the 1st Moment of Adam. \citep{adafactorICML2018}
    \item $\ell_2$  Rank-1 --- After each update, we perform an SVD decomposition of the auxiliary variable, and only keep the top singular value and its corresponding vectors. During the subsequent update, the auxiliary variable is reconstructed via an outer product. Unlike the NMF Rank-1 Approximation, this approach is not limited to non-negative values, but it is extremely slow and cannot be used in practice.
    \item Count-Sketch --- As described in Section \ref{sec:algorithm}. This approach is also not limited to non-negative values and is capable of compressing the auxiliary variables for all optimizers efficiently.
\end{enumerate}

\begin{table} [ht]
\caption{Abbreviations}
\begin{center}
    \begin{tabular}{ |r|r| } 
    \hline
    Title & Symbol \\
    \hline
    Count-Sketch & CS \\
    Low-Rank & LR \\
    Adam 1st Moment & M \\
    Adam 2nd Moment & V \\
    Non-Negative Matrix Factorization & NNF \\
    \hline
    \end{tabular}
\end{center}
\label{fig:abbv}
\end{table}

\subsection{Small-Scale Experiments}
{\bf Wikitext-2:} The language model is a 2-layer LSTM with 672 hidden units. The dimensionality of the word embeddings is equal to the number of hidden units. The model is unrolled 35 steps for the back-propagation through time (BPTT). The model is regularized via Dropout with a 50\% chance of disabling a unit. We train the model for 40 epochs with a mini-batch size of 20. For Momentum, the learning rate is 2.5, the decay rate $\gamma$ is 0.9, and we clip the gradient norm to 0.25. For Adam, the learning rate is 0.001, the beta values $\beta_1, \beta_2$ are (0.9, 0.999), and gradient clipping is 1. We reduce the learning rate by $4\times$ whenever the validation error plateaus. We use the full softmax layer, so only the embedding layer is sparse for this dataset.

{\bf $\ell_2$ -Norm Approximation Error:} Fig.~\ref{fig:l2_dist} shows the $\ell_2$-Norm between the approximation and the original auxiliary variable over several training iterations. The left figure is for the Momentum optimizer, while the right figure is for the 2nd Moment for the Adam optimizer. All of the methods are given roughly an equal amount of parameters to approximate the original auxiliary variable. For the Wikitext-2 dataset, the embedding and softmax layers use [33,278, 256] matrices. Therefore, the rank-1 decomposition uses two vectors that use 33,278 + 256 = 33,534 parameters. The count-sketch data structure is represented with a [3, 16, 672] tensor, containing 32,256 parameters. Our count-sketch approach maps the 33,278 word vocabulary into 16 distinct bins, so there are about 2,080 collisions for each bucket.

 The Adam optimizer's 2nd Moment is strictly non-negative and is suitable for the NMF Rank-1 approximation. For the Momentum variable, we supplement the NMF decomposition with the $\ell_2$ SVD decomposition. The $\ell_2$ SVD decomposition maintains a good approximation of the Momentum variable. However, it is extremely slow during training, so we only show the approximation error for the first epoch of training. As expected, the NMF Rank-1 baseline poorly approximates the momentum variable, which is not strictly non-negative. It experiences significant variance in its approximation quality. The Count-Sketch is a consistent estimator for both variables with slightly more error for both variables. 

{\bf Test Perplexity:} Tables~\ref{fig:wiki2_mtm_perf},\ref{fig:wiki2_adam_perf} show the test perplexity after training the model with the Momentum and Adam optimizers. For the momentum optimizer, the NNM Low-Rank approximation performs poorly, reinforcing the results from Fig.~\ref{fig:l2_dist}. When only the 2nd moment is compressed, the NNM Low-Rank and Count-Sketch approximations have negligible differences. When we compress both the 1st and 2nd moments with the Count-Sketch, there is some minor accuracy loss from the original optimizer.

\begin{table}
\caption{Test Perplexity for Momentum Optimizer on the Wikitext-2 dataset. The size of the count-sketch tensor is [3, 16, 672] while the rank-1 approximation uses 33,278 + 672 parameters.}
\begin{center}
    \begin{tabular}{ |r|r|r| } 
    \hline
    Momentum & CS & LR-NMF \\
    \hline
    94.25 & 95.93 & 176.31 \\
    \hline
    \end{tabular}
\end{center}
\label{fig:wiki2_mtm_perf}
\end{table}

\begin{table}
\caption{Test Perplexity for Adam Optimizer on the Wikitext-2 dataset. The modifiers indicate which auxiliary variables are compressed.}
\begin{center}
    \begin{tabular}{ |r|r|r|r| } 
    \hline
     CS-MV & Adam & CS-V & LR-NMF-V \\
    \hline
    109.24 & 105.14 & 106.32 & 106.21 \\
    \hline
    \end{tabular}
\end{center}
\label{fig:wiki2_adam_perf}
\end{table}

\begin{figure} [ht]
\begin{center}
\mbox{
\includegraphics[width=0.233\textwidth]{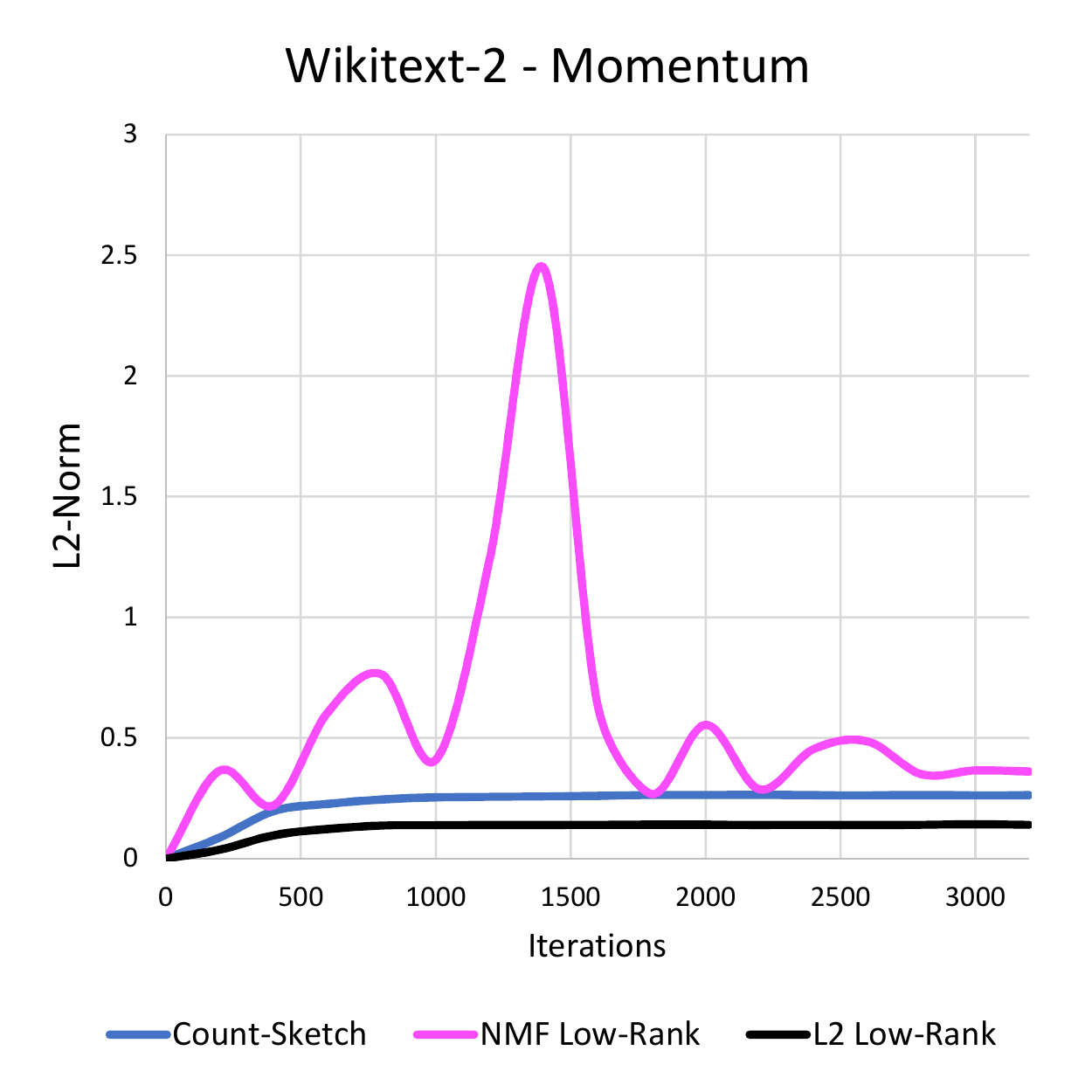}
\includegraphics[width=0.233\textwidth]{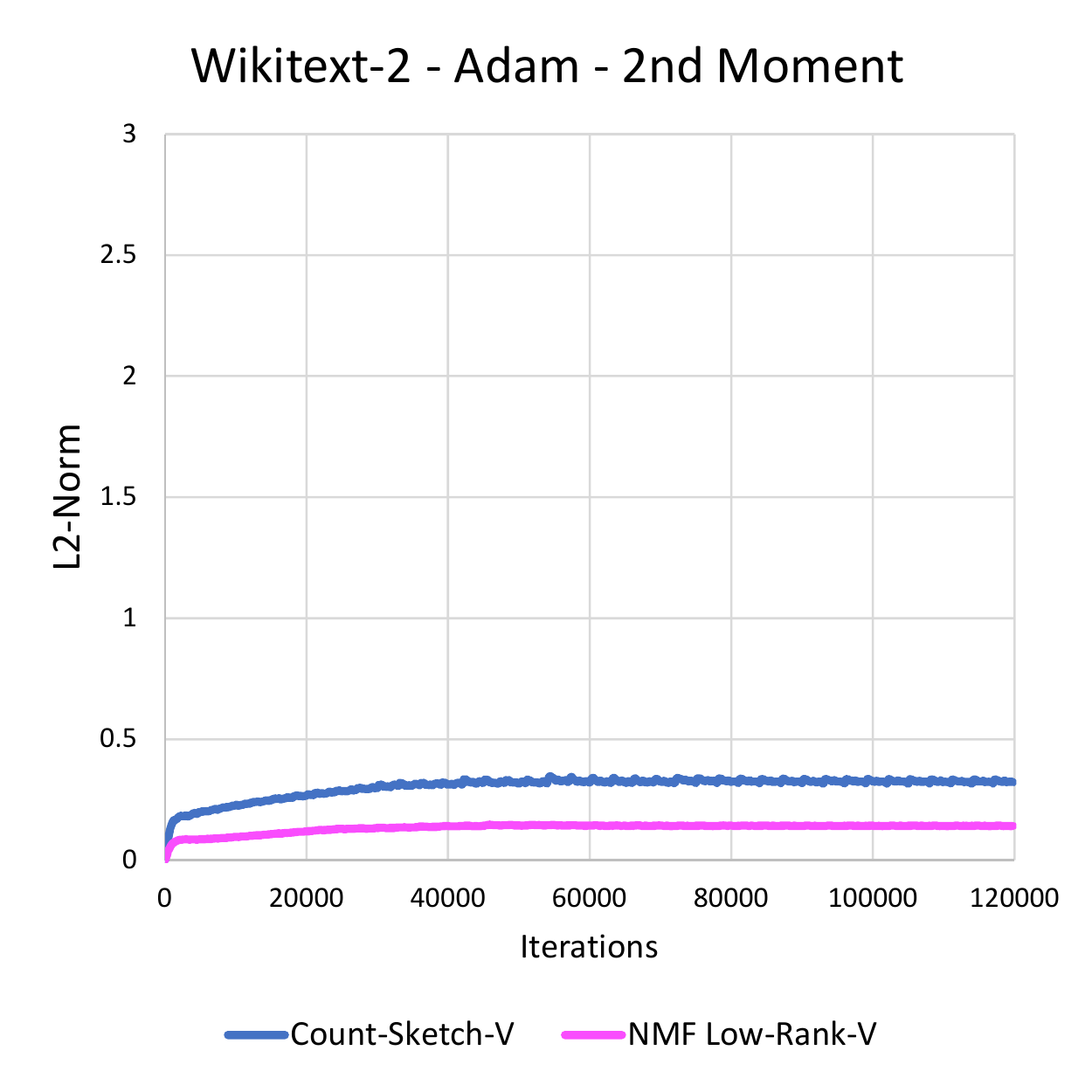}
}
\end{center}
\caption{Left - Momentum, Right - Adam - 2nd Moment. $\ell_2$-Norm between the approximation and the original auxiliary variable.}
\label{fig:l2_dist}
\end{figure}

\textbf{MegaFace:} For this experiment, we obtain pretrained embeddings of size 512 from the FaceNet architecture \cite{schroff2015facenet} trained on the MS-Celeb-1M dataset \footnote{\url{https://github.com/davidsandberg/facenet}}. Afterwards, we train a softmax classifier on the MegaFace dataset using LSH Sampling \cite{pmlr-v80-yen18a, vijayanarasimhan2014deep}. For LSH Sampling, we use SimHash --- Signed Random Projection (SRP) with K=15 bits per hash fingerprint. There are L=16 hash tables that are rebuilt every 250 iterations. For Adam, the learning rate is 0.001 and the beta values $\beta_1, \beta_2$ are (0.9, 0.999). For Adagrad, the learning rate is 0.1. All the models were trained for 10 epochs.

Fig.~\ref{fig:cleaning} shows the effect of cleaning the Count-Min Sketch Tensor on its corresponding optimizer. We measure how the testing accuracy, convergence rate, and auxiliary variable error changes because of cleaning for the Adam and Adagrad optimizers. The Count-Min Sketch tensor is set to $20\%$ of the original variable's size. For Adam, the cleaning scheme is every 125 iterations, multiply the count-min sketch by a constant $0.2$. For Adagrad, the rate of cleaning is the same, but the constant is changed to $0.5$.

For both Adam and Adagrad, there is a noticeable drop in $\ell_2$-Norm error with cleaning, which reflects positively in terms of test accuracy and convergence. For Adam, the count-sketch optimizer with cleaning closely matches the convergence rate of the baseline and slightly surpasses its test accuracy. The test accuracy for Count-Sketch with cleaning is 69.4\%, while the baseline is 69.03\%. For Adagrad, cleaning did not improve the initial convergence rate, but allowed the final test accuracy to match the baseline. There is a solid $1\%$ improvement in test accuracy from $68.37\%$ to $69.28\%$ by using cleaning for the Count-Sketch Adagrad optimizer.

Given that the Adam optimizer already contains an exponential decay term, it is surprising that cleaning is necessary. However, despite further hyper-parameter tuning, the count-sketch optimizer with cleaning still achieves the best performance. For dense gradients, the decay term is applied to all elements. Since the gradients are sparse, only the non-zero elements are updated. Thus, the decay is applied in an irregular fashion for the elements in the sketch.

% Absolute Error over Iterations - Periodic Cleaning, Incremental Cleaning, No Cleaning, Low-Rank - Adam and Adagrad
% Accuracy over Iterations - Periodic Cleaning, Incremental Cleaning, No Cleaning, Low-Rank - Adam and Adagrad
% Accuracy over Wall-Clock Time - Sparse, Count-Sketch, Low-Rank
% Accuracy over Epochs - Sparse, Count-Sketch, Low-Rank

\begin{figure*} [ht]
\begin{center}
\mbox{
\includegraphics[width=0.2425\textwidth]{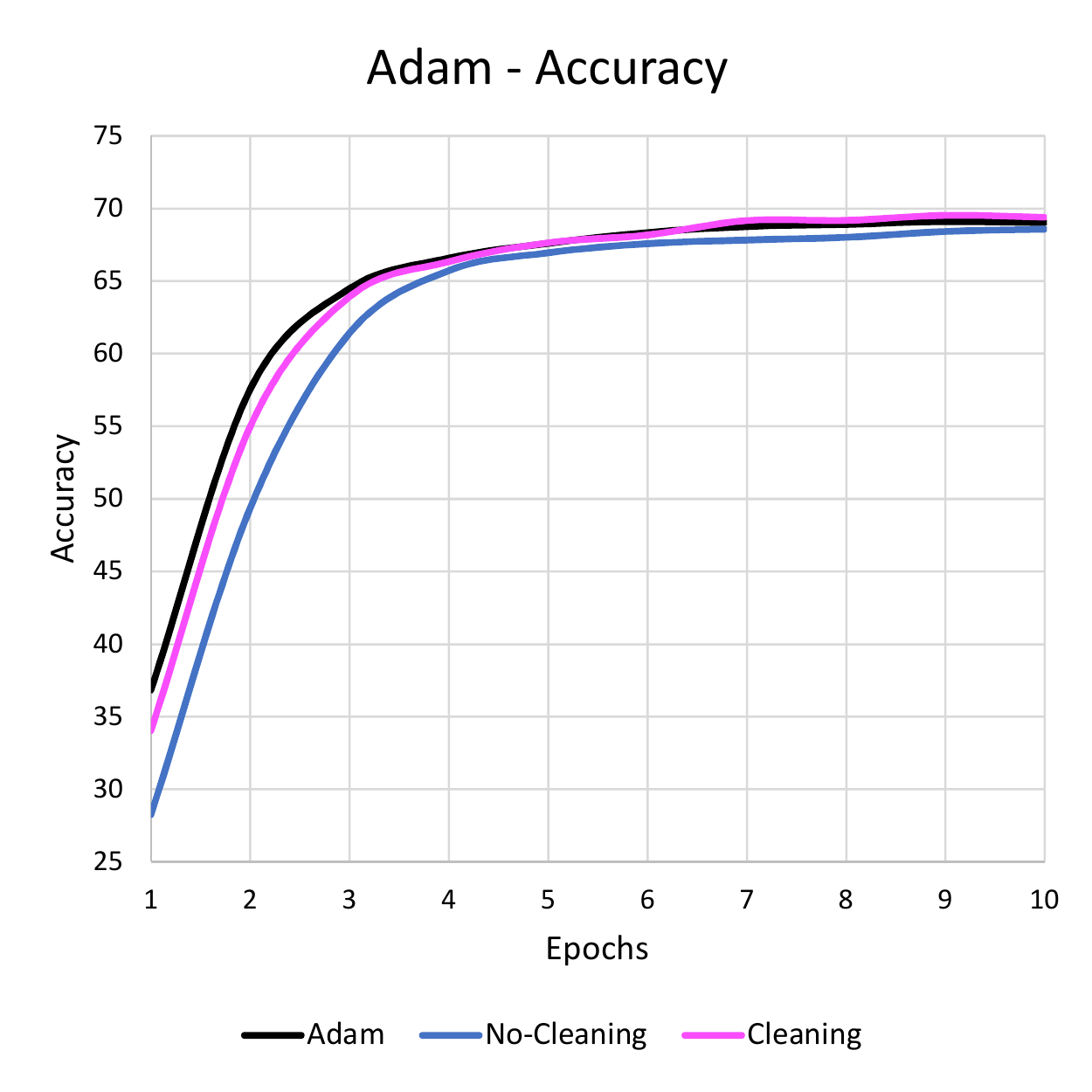}
\includegraphics[width=0.2425\textwidth]{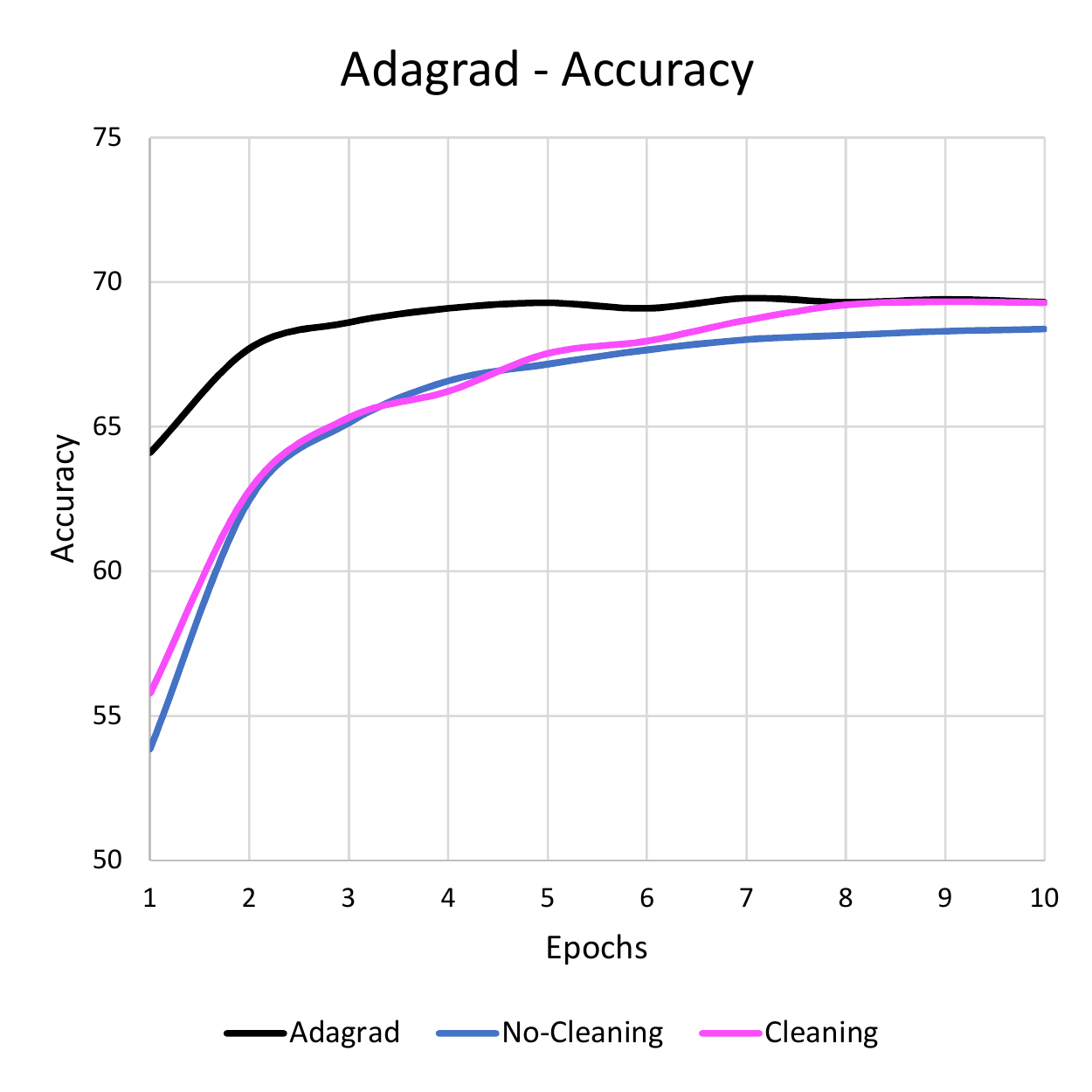}
\includegraphics[width=0.2425\textwidth]{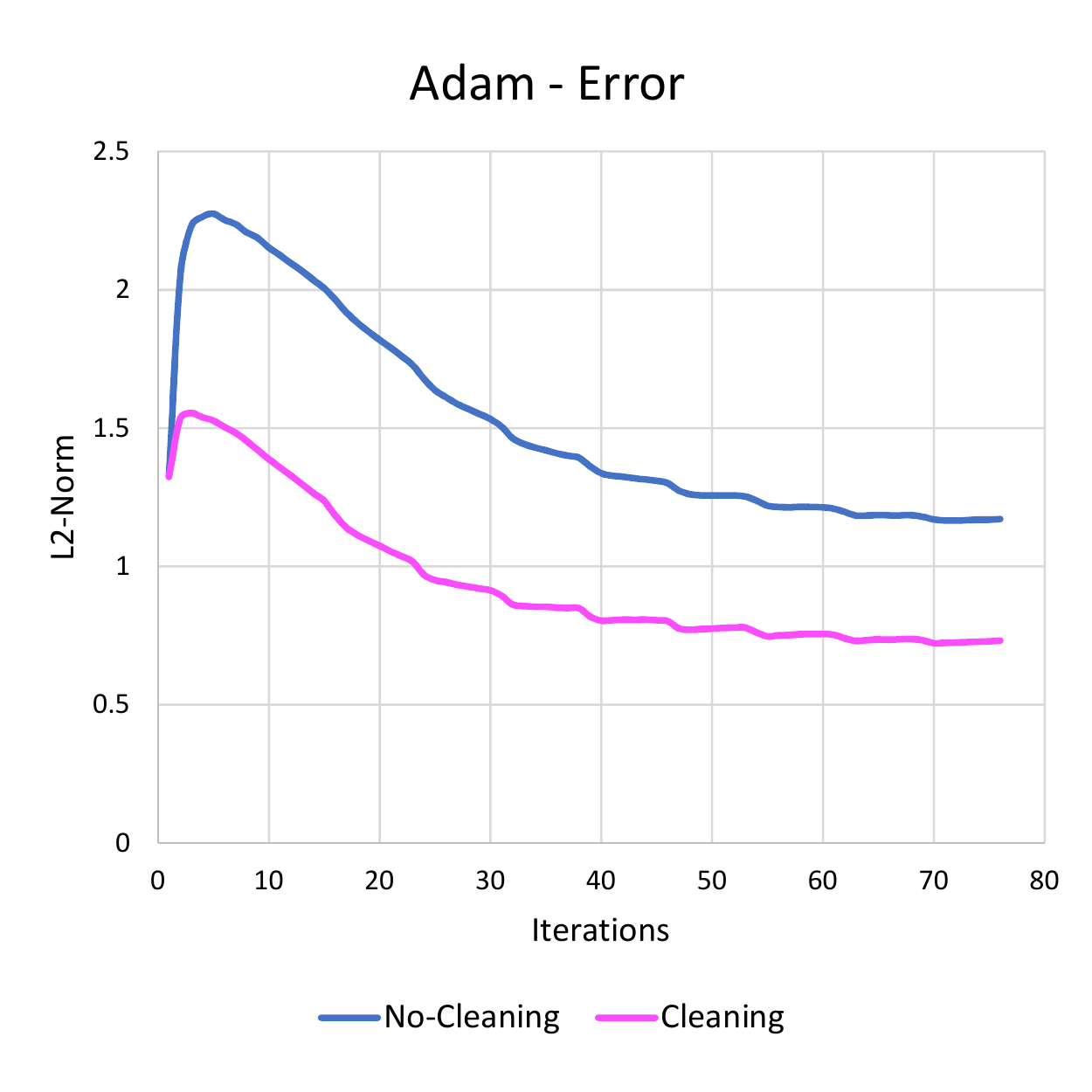}
\includegraphics[width=0.2425\textwidth]{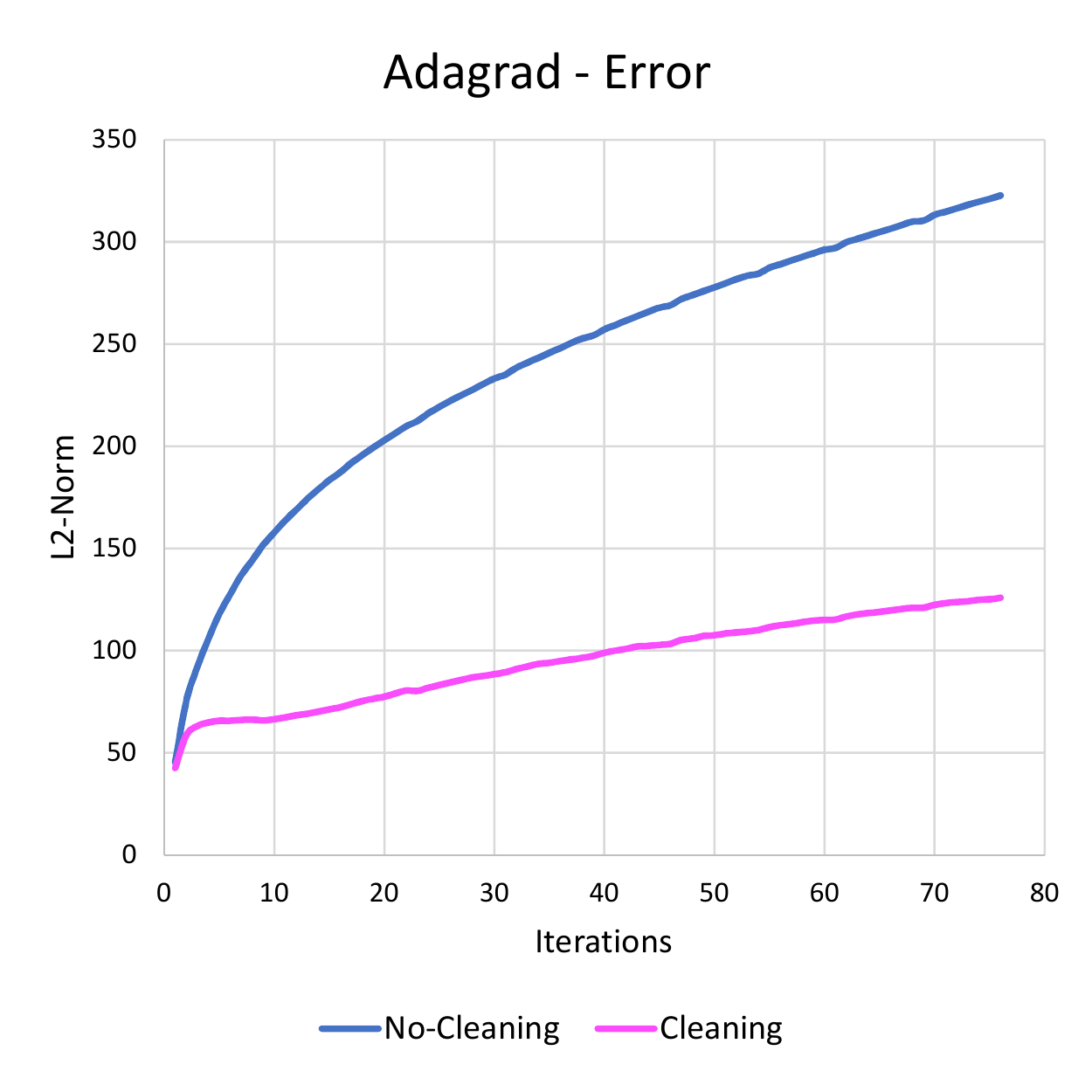}
}
\end{center}
\caption{The effect of cleaning on the Count-Min Sketch Tensor and its corresponding optimizer for the MegaFace dataset.}
\label{fig:cleaning}
\end{figure*}

\subsection{Large-Scale Language Model}
Since the Wikitext-103 and LM1B datasets have large vocabularies, we use Sampled Softmax \citep{jean2014sampled} to induce sparsity in the softmax layer and for faster training. Each Count-Sketch Tensor is $5 \times$ smaller than the original variable. Therefore, there are at least 15 collisions for each bin on average.

{\bf Adagrad - Wikitext-103:} Our language model is a single layer LSTM with 1024 hidden units. The dimensionality of the word embeddings is 256 and we use a projection layer between the LSTM and Softmax layers. The model is unrolled 35 steps BPTT. The model is regularized via Dropout with $p=0.25$. We train the model for 25 epochs with a mini-batch size of 1024. For the Adagrad optimizer, the gradient norm is clipped to 0.1, the learning rate starts at 0.4 and decays linearly to 0 during training.

{\bf Results:} For the Wikitext-103 dataset, we allocated a [3, 17,849, 256] Count-Sketch tensor for each auxiliary variable. By providing the Count-Sketch with more parameters, our method has notably better test accuracy than the NMF low-rank approximation while using only slightly more memory. In addition, despite using more parameters than the low-rank approximation, the count-sketch optimizer is still somewhat faster. Finally, the low-rank approximation fails to meet the same accuracy as the original baseline, while surprisingly the count-sketch optimizer has the best test perplexity.

\begin{table}
\caption{Test Perplexity, Running Time, and Memory Consumption on the Wikitext-103 dataset using the Adagrad Optimizer.  CS --- Count-Sketch, LR --- Low-Rank}
\begin{center}
    \begin{tabular}{ |r|r|r|r| } 
    \hline
     Metric & Adagrad & CS & LR-NMF \\
    \hline
    Time & {\bf 6.4} & 6.6 & 6.7 \\
    Size & 10,625 & 10,089 & {\bf 10,077} \\
    Test Perplexity & 57.63 & {\bf 56.07} & 58.27 \\
    \hline
    \end{tabular}
\end{center}
\label{fig:adagrad_perf}
\end{table}

{\bf Adam - LM1B:} For the 1-Billion Word dataset, our goal is to mimic multi-GPU distributed training on a single GPU. The original batch size is 128 with a learning rate of 5e-4. By increasing our batch size from 128 to 1024, we scale our learning rate linearly by $8\times$ \citep{goyal2017scaling}. In addition, we decay our learning rate linearly to zero over 5 training epochs. We double the LSTM size from 1024 to 2048, but keep the word embedding size at 256. The model is unrolled 20 steps BPTT. Dropout is kept nominally at $p=0.01$ and the gradient norm is clipped to 1. A surprising side effect of increasing the batch size was that we reduced our training time by roughly $2\times$ from 12.25 hours to 6.25 hours per epoch despite using a single GPU.

{\bf Results:} For the 1-Billion Word dataset, we allocated a [3, 52,898, 256] Count-Sketch tensor for each auxiliary variable. Our primary comparison is only with the 2nd moment because the NMF low-rank approximation is not applicable to the 1st moment. The count-sketch is slightly more accurate than the low-rank approximation. When both the 1st and 2nd moments are compressed with the count-sketch tensor, its accuracy is on-par with the low-rank approximation that compresses only the 2nd moment. In general, the count-sketch tensor is $8\%$ faster than the low-rank approach while using substantially less GPU memory. For large matrices, there is a noticeable cost with reconstructing the entire matrix to update only a sparse subset of values.

\begin{table} [ht]
\caption{Running Time and Memory Consumption on the 1-Billion Word dataset for the Adam optimizer.}
\begin{center}
    \begin{tabular}{ |r|r|r|r|r| } 
    \hline
     Metric & CS-MV & Adam & CS-V & LR-NMF-V \\
    \hline
    Time & 27.1 & {\bf 26.4} & 26.75 & 29.2 \\
    Size & {\bf 8,591} & 11,707 & 10,167 & 13,259 \\
    \hline
    \end{tabular}
\end{center}
\label{fig:perf}
\end{table}
%11,707 - sparse
%13,399 - dense

\begin{table} [ht]
\caption{Convergence Rate (Test Perplexity) after 5 epochs on the 1-Billion Word dataset. The modifiers indicate which auxiliary variables are compressed for the Adam optimizer.}
\begin{center}
    \begin{tabular}{ |r|r|r|r|r| } 
    \hline
     Epoch & CS-MV & Adam & CS-V & LR-NMF-V \\
    \hline
    1 & 50.78 & 48.48 & 49.49 & 50.04 \\
    2 & 46.08 & 45.34 & 45.22 & 45.60 \\
    3 & 43.71 & 42.79 & 42.95 & 43.55 \\
    4 & 41.82 & 41.15 & 41.23 & 41.82 \\
    5 & 40.55 & 39.90 & {\bf 39.88} & 40.41 \\
    \hline
    \end{tabular}
\end{center}
\label{fig:convergence}
\end{table}

\subsection{Extreme Classification}
\label{sec:amz}
For the extremely large-scale classification task, we conducted our experiments on an Amazon recommendation dataset. The task is to predict an object out of over 49 million classes given a query. The text query is parsed into trigram features. Feature hashing is applied to convert the strings into integers. The input feature dimension is 80K. On average, there are on 30 non-zero features per query, so the input layer is very sparse and suitable for our Count-Sketch optimizer. We trained a single hidden layer, fully-connected neural network with an embedding dimension of 1024.

A traditional softmax classifier would require over 200 GB of memory, which is well beyond the memory capacity of the largest GPUs. Instead, we leverage a novel approach for extreme classification called Merged-Averaged Classifiers via Hashing (MACH) \citep{huang2018mach}. This algorithm randomly merges the output classes into a manageable number of coarse-grained, meta-classes via universal hashing. Several independent, fully-connected neural networks are trained to solve this meta-class classification task. Each meta-classifier is associated with a unique hash function that creates a distinct class mapping. At inference time, we recover the scores for the original classes by aggregating the meta-class scores assigned to the original output class. For this experiment, we used 20K meta-classes in the output layer of each meta-classifier. For high-accuracy models, we use 32 meta-classifiers. Each individual meta-classifier required 414 MB of memory for a total of 12.95 GB. Therefore, our ensemble MACH classifier used $15 \times$ less memory than a monolithic softmax classifier.

Since we are primarily interested in faster training times, we limit ourselves to 4 meta-classifiers in this experiment. For our baseline, each meta-classifier is trained using the Adam optimizer with a batch size of 750. Given these settings, a single meta-classifier takes 4 GB of GPU memory, allowing us to train 4 models in parallel on a single GPU. For maximum memory savings, we eliminate the 1st moment and use a count-min sketch tensor of size [3, 266, 1024] for the 2nd moment \textbf{(1\% of original size)}. By using the Adam Count-Sketch optimizer, we reduce the memory cost for each model from 4 GB to 2.6 GB \textbf{(45\% smaller)}. We take of advantage of this extra memory by increasing the batch size from 750 to 2600 \textbf{(3.5$\times$ larger)}. As a result, the running time per epoch decreased from 5.32 hours to 3.3 hours \textbf{(38\% faster)}.

We measure the accuracy of the MACH model using the Recall@100 metric on a test dataset containing 20K queries. First, we evaluate the meta-classifiers and aggregate their scores. Then, we check how often the target class appears within the top 100 scores generated by the classifier. A major bottleneck during evaluation is sorting the 49.5 million classes to find the top 100 scores. Since we are only comparing the model's relative performance and are interested in fast running times, we down-sample the scores from 49.5 million to 1 million. The class subset contains the target classes for all 20K test queries and a random sample of the remaining classes. Given 16 meta-classifiers, the Adam baseline has a 0.6881 recall, while the Count-Sketch optimizer achieves a 0.6889 recall.

\begin{table} [ht]
\caption{Extreme Classification --- A MACH ensemble with 4 meta-classifiers is trained on a single GPU using Adam and the Count-Sketch optimizer.}
\begin{center}
    \begin{tabular}{ |l|l|l|l| } 
    \hline
    Type & Batch Size & Epoch Time & Recall@100  \\
    \hline
    Adam & 750 & 5.32 & 0.4704\\
    CS-V & \textbf{2600} & \textbf{3.3} & \textbf{0.4789}\\
    \hline
    \end{tabular}
\end{center}
\label{fig:mach}
\end{table}

% CS-V
% 2 - 0.2460
% 4 - 0.4774
% 8 - 0.6173
% 16 - 0.6889

% Adam
% 2 - 0.2381
% 4 - 0.4749
% 8 - 0.6169
% 16 - 0.6881

\section{Conclusion and Future Work}
In this paper, we present the concept of a count-sketch tensor to compress the auxiliary variables associated with popular first-order optimizers. The count-sketch tensor retains the constant-time update and query operations, while maintaining structured sparsity for high-speed vectorized operations. The count-sketch tensor can reduce the memory usage of large-scale models with minimal cost by taking advantage of the model's sparsity. 
Going forward, we are interested in compressing the auxiliary variables associated with the hidden layers without incurring any performance penalty. We hope to leverage recent ideas of adding sparsity to the hidden layers in order to increase the size of the model without increasing its computational cost \citep{spring2017lsh, shazeer2017moe, wen2017learning}. Structured sparsity in the hidden layers would mesh well with our current approach for the Embedding and Softmax layers.

% (e.g. Locality-Sensitive Hashing, Mixture of Experts, Block Sparsity)

\balance
\renewcommand*{\bibfont}{\normalsize}
\bibliography{bibliography}


\begin{thebibliography}{00}

%%% the Web address.

\ifx \showCODEN    \undefined \def \showCODEN     #1{\unskip}     \fi
\ifx \showDOI      \undefined \def \showDOI       #1{#1}\fi
\ifx \showISBNx    \undefined \def \showISBNx     #1{\unskip}     \fi
\ifx \showISBNxiii \undefined \def \showISBNxiii  #1{\unskip}     \fi
\ifx \showISSN     \undefined \def \showISSN      #1{\unskip}     \fi
\ifx \showLCCN     \undefined \def \showLCCN      #1{\unskip}     \fi
\ifx \shownote     \undefined \def \shownote      #1{#1}          \fi
\ifx \showarticletitle \undefined \def \showarticletitle #1{#1}   \fi
\ifx \showURL      \undefined \def \showURL       {\relax}        \fi
% invisible to TeX
\providecommand\bibfield[2]{#2}
\providecommand\bibinfo[2]{#2}
\providecommand\natexlab[1]{#1}
\providecommand\showeprint[2][]{arXiv:#2}

\bibitem[\protect\citeauthoryear{Aghazadeh, Spring, Lejeune, Dasarathy,
  Shrivastava, and richard baraniuk}{Aghazadeh et~al\mbox{.}}{2018}]%
        {missionICML2018}
\bibfield{author}{\bibinfo{person}{Amirali Aghazadeh}, \bibinfo{person}{Ryan
  Spring}, \bibinfo{person}{Daniel Lejeune}, \bibinfo{person}{Gautam
  Dasarathy}, \bibinfo{person}{Anshumali Shrivastava}, {and}
  \bibinfo{person}{richard baraniuk}.} \bibinfo{year}{2018}\natexlab{}.
\newblock \showarticletitle{{MISSION}: Ultra Large-Scale Feature Selection
  using Count-Sketches}. In \bibinfo{booktitle}{{\em Proceedings of the 35th
  International Conference on Machine Learning}} {\em
  (\bibinfo{series}{Proceedings of Machine Learning Research})},
  \bibfield{editor}{\bibinfo{person}{Jennifer Dy} {and}
  \bibinfo{person}{Andreas Krause}} (Eds.), Vol.~\bibinfo{volume}{80}.
  \bibinfo{publisher}{PMLR}, \bibinfo{address}{Stockholmsmässan, Stockholm
  Sweden}, \bibinfo{pages}{80--88}.
\newblock
\showURL{%
\url{http://proceedings.mlr.press/v80/aghazadeh18a.html}}


\bibitem[\protect\citeauthoryear{Charikar, Chen, and Farach-Colton}{Charikar
  et~al\mbox{.}}{2002}]%
        {charikar2002finding}
\bibfield{author}{\bibinfo{person}{M. Charikar}, \bibinfo{person}{K. Chen},
  {and} \bibinfo{person}{M. Farach-Colton}.} \bibinfo{year}{2002}\natexlab{}.
\newblock \showarticletitle{Finding frequent items in data streams}. In
  \bibinfo{booktitle}{{\em Intl. Colloquium on Automata, Languages, and
  Programming}}. Springer, \bibinfo{pages}{693--703}.
\newblock


\bibitem[\protect\citeauthoryear{Chelba, Mikolov, Schuster, Ge, Brants, Koehn,
  and Robinson}{Chelba et~al\mbox{.}}{2013}]%
        {chelba2013gbw}
\bibfield{author}{\bibinfo{person}{Ciprian Chelba}, \bibinfo{person}{Tomas
  Mikolov}, \bibinfo{person}{Mike Schuster}, \bibinfo{person}{Qi Ge},
  \bibinfo{person}{Thorsten Brants}, \bibinfo{person}{Phillipp Koehn}, {and}
  \bibinfo{person}{Tony Robinson}.} \bibinfo{year}{2013}\natexlab{}.
\newblock \showarticletitle{One billion word benchmark for measuring progress
  in statistical language modeling}.
\newblock \bibinfo{journal}{{\em arXiv preprint arXiv:1312.3005\/}}
  (\bibinfo{year}{2013}).
\newblock


\bibitem[\protect\citeauthoryear{Chen, Xu, Zhang, and Guestrin}{Chen
  et~al\mbox{.}}{2016}]%
        {chen2016training}
\bibfield{author}{\bibinfo{person}{Tianqi Chen}, \bibinfo{person}{Bing Xu},
  \bibinfo{person}{Chiyuan Zhang}, {and} \bibinfo{person}{Carlos Guestrin}.}
  \bibinfo{year}{2016}\natexlab{}.
\newblock \showarticletitle{Training deep nets with sublinear memory cost}.
\newblock \bibinfo{journal}{{\em arXiv preprint arXiv:1604.06174\/}}
  (\bibinfo{year}{2016}).
\newblock


\bibitem[\protect\citeauthoryear{Cormode and Muthukrishnan}{Cormode and
  Muthukrishnan}{2005}]%
        {cormode2005improved}
\bibfield{author}{\bibinfo{person}{Graham Cormode} {and} \bibinfo{person}{Shan
  Muthukrishnan}.} \bibinfo{year}{2005}\natexlab{}.
\newblock \showarticletitle{An improved data stream summary: the count-min
  sketch and its applications}.
\newblock \bibinfo{journal}{{\em Journal of Algorithms\/}}
  \bibinfo{volume}{55}, \bibinfo{number}{1} (\bibinfo{year}{2005}),
  \bibinfo{pages}{58--75}.
\newblock


\bibitem[\protect\citeauthoryear{Devlin, Chang, Lee, and Toutanova}{Devlin
  et~al\mbox{.}}{2018}]%
        {devlin2018bert}
\bibfield{author}{\bibinfo{person}{Jacob Devlin}, \bibinfo{person}{Ming-Wei
  Chang}, \bibinfo{person}{Kenton Lee}, {and} \bibinfo{person}{Kristina
  Toutanova}.} \bibinfo{year}{2018}\natexlab{}.
\newblock \showarticletitle{BERT: Pre-training of Deep Bidirectional
  Transformers for Language Understanding}.
\newblock \bibinfo{journal}{{\em arXiv preprint arXiv:1810.04805\/}}
  (\bibinfo{year}{2018}).
\newblock


\bibitem[\protect\citeauthoryear{Duchi, Hazan, and Singer}{Duchi
  et~al\mbox{.}}{2011}]%
        {duchi2011adagrad}
\bibfield{author}{\bibinfo{person}{John Duchi}, \bibinfo{person}{Elad Hazan},
  {and} \bibinfo{person}{Yoram Singer}.} \bibinfo{year}{2011}\natexlab{}.
\newblock \showarticletitle{Adaptive subgradient methods for online learning
  and stochastic optimization}.
\newblock \bibinfo{journal}{{\em Journal of Machine Learning Research\/}}
  \bibinfo{volume}{12}, \bibinfo{number}{Jul} (\bibinfo{year}{2011}),
  \bibinfo{pages}{2121--2159}.
\newblock


\bibitem[\protect\citeauthoryear{Goyal, Doll{\'a}r, Girshick, Noordhuis,
  Wesolowski, Kyrola, Tulloch, Jia, and He}{Goyal et~al\mbox{.}}{2017}]%
        {goyal2017scaling}
\bibfield{author}{\bibinfo{person}{Priya Goyal}, \bibinfo{person}{Piotr
  Doll{\'a}r}, \bibinfo{person}{Ross Girshick}, \bibinfo{person}{Pieter
  Noordhuis}, \bibinfo{person}{Lukasz Wesolowski}, \bibinfo{person}{Aapo
  Kyrola}, \bibinfo{person}{Andrew Tulloch}, \bibinfo{person}{Yangqing Jia},
  {and} \bibinfo{person}{Kaiming He}.} \bibinfo{year}{2017}\natexlab{}.
\newblock \showarticletitle{Accurate, large minibatch SGD: training imagenet in
  1 hour}.
\newblock \bibinfo{journal}{{\em arXiv preprint arXiv:1706.02677\/}}
  (\bibinfo{year}{2017}).
\newblock


\bibitem[\protect\citeauthoryear{Han, Mao, and Dally}{Han
  et~al\mbox{.}}{2015}]%
        {han2015deep}
\bibfield{author}{\bibinfo{person}{Song Han}, \bibinfo{person}{Huizi Mao},
  {and} \bibinfo{person}{William~J Dally}.} \bibinfo{year}{2015}\natexlab{}.
\newblock \showarticletitle{Deep compression: Compressing deep neural networks
  with pruning, trained quantization and huffman coding}.
\newblock \bibinfo{journal}{{\em arXiv preprint arXiv:1510.00149\/}}
  (\bibinfo{year}{2015}).
\newblock


\bibitem[\protect\citeauthoryear{Hoffer, Hubara, and Soudry}{Hoffer
  et~al\mbox{.}}{2017}]%
        {largebatchNIPS2017}
\bibfield{author}{\bibinfo{person}{Elad Hoffer}, \bibinfo{person}{Itay Hubara},
  {and} \bibinfo{person}{Daniel Soudry}.} \bibinfo{year}{2017}\natexlab{}.
\newblock \showarticletitle{Train longer, generalize better: closing the
  generalization gap in large batch training of neural networks}.
\newblock In \bibinfo{booktitle}{{\em Advances in Neural Information Processing
  Systems 30}}, \bibfield{editor}{\bibinfo{person}{I.~Guyon},
  \bibinfo{person}{U.~V. Luxburg}, \bibinfo{person}{S.~Bengio},
  \bibinfo{person}{H.~Wallach}, \bibinfo{person}{R.~Fergus},
  \bibinfo{person}{S.~Vishwanathan}, {and} \bibinfo{person}{R.~Garnett}}
  (Eds.). \bibinfo{publisher}{Curran Associates, Inc.},
  \bibinfo{pages}{1731--1741}.
\newblock


\bibitem[\protect\citeauthoryear{Huang, Wang, Medini, and Shrivastava}{Huang
  et~al\mbox{.}}{2018}]%
        {huang2018mach}
\bibfield{author}{\bibinfo{person}{Qixuan Huang}, \bibinfo{person}{Yiqiu Wang},
  \bibinfo{person}{Tharun Medini}, {and} \bibinfo{person}{Anshumali
  Shrivastava}.} \bibinfo{year}{2018}\natexlab{}.
\newblock \showarticletitle{Extreme Classification in Log Memory}.
\newblock \bibinfo{journal}{{\em arXiv preprint arXiv:1810.04254\/}}
  (\bibinfo{year}{2018}).
\newblock


\bibitem[\protect\citeauthoryear{Jean, Cho, Memisevic, and Bengio}{Jean
  et~al\mbox{.}}{2014}]%
        {jean2014sampled}
\bibfield{author}{\bibinfo{person}{S{\'e}bastien Jean},
  \bibinfo{person}{Kyunghyun Cho}, \bibinfo{person}{Roland Memisevic}, {and}
  \bibinfo{person}{Yoshua Bengio}.} \bibinfo{year}{2014}\natexlab{}.
\newblock \showarticletitle{On using very large target vocabulary for neural
  machine translation}.
\newblock \bibinfo{journal}{{\em arXiv preprint arXiv:1412.2007\/}}
  (\bibinfo{year}{2014}).
\newblock


\bibitem[\protect\citeauthoryear{Jozefowicz, Vinyals, Schuster, Shazeer, and
  Wu}{Jozefowicz et~al\mbox{.}}{2016}]%
        {jozefowicz2016limits}
\bibfield{author}{\bibinfo{person}{Rafal Jozefowicz}, \bibinfo{person}{Oriol
  Vinyals}, \bibinfo{person}{Mike Schuster}, \bibinfo{person}{Noam Shazeer},
  {and} \bibinfo{person}{Yonghui Wu}.} \bibinfo{year}{2016}\natexlab{}.
\newblock \bibinfo{title}{Exploring the limits of language modeling}.
\newblock   (\bibinfo{year}{2016}).
\newblock
\showURL{%
\url{https://arxiv.org/pdf/1602.02410.pdf}}


\bibitem[\protect\citeauthoryear{Kingma and Ba}{Kingma and Ba}{2014}]%
        {kingma2014adam}
\bibfield{author}{\bibinfo{person}{Diederik~P Kingma} {and}
  \bibinfo{person}{Jimmy Ba}.} \bibinfo{year}{2014}\natexlab{}.
\newblock \showarticletitle{Adam: A method for stochastic optimization}.
\newblock \bibinfo{journal}{{\em arXiv preprint arXiv:1412.6980\/}}
  (\bibinfo{year}{2014}).
\newblock


\bibitem[\protect\citeauthoryear{Krause, Lu, Murray, and Renals}{Krause
  et~al\mbox{.}}{2016}]%
        {krause2016mlstm}
\bibfield{author}{\bibinfo{person}{Ben Krause}, \bibinfo{person}{Liang Lu},
  \bibinfo{person}{Iain Murray}, {and} \bibinfo{person}{Steve Renals}.}
  \bibinfo{year}{2016}\natexlab{}.
\newblock \showarticletitle{Multiplicative LSTM for sequence modelling}.
\newblock \bibinfo{journal}{{\em arXiv preprint arXiv:1609.07959\/}}
  (\bibinfo{year}{2016}).
\newblock


\bibitem[\protect\citeauthoryear{Matusevych, Smola, and Ahmed}{Matusevych
  et~al\mbox{.}}{2012}]%
        {matusevych2012hokusai}
\bibfield{author}{\bibinfo{person}{Sergiy Matusevych}, \bibinfo{person}{Alex
  Smola}, {and} \bibinfo{person}{Amr Ahmed}.} \bibinfo{year}{2012}\natexlab{}.
\newblock \showarticletitle{Hokusai-sketching streams in real time}.
\newblock \bibinfo{journal}{{\em arXiv preprint arXiv:1210.4891\/}}
  (\bibinfo{year}{2012}).
\newblock


\bibitem[\protect\citeauthoryear{Merity, Xiong, Bradbury, and Socher}{Merity
  et~al\mbox{.}}{2016}]%
        {merity2016pointer}
\bibfield{author}{\bibinfo{person}{Stephen Merity}, \bibinfo{person}{Caiming
  Xiong}, \bibinfo{person}{James Bradbury}, {and} \bibinfo{person}{Richard
  Socher}.} \bibinfo{year}{2016}\natexlab{}.
\newblock \showarticletitle{Pointer sentinel mixture models}.
\newblock \bibinfo{journal}{{\em arXiv preprint arXiv:1609.07843\/}}
  (\bibinfo{year}{2016}).
\newblock


\bibitem[\protect\citeauthoryear{Micikevicius, Narang, Alben, Diamos, Elsen,
  Garcia, Ginsburg, Houston, Kuchaiev, Venkatesh, and Wu}{Micikevicius
  et~al\mbox{.}}{2018}]%
        {micikevicius2018mixed}
\bibfield{author}{\bibinfo{person}{Paulius Micikevicius},
  \bibinfo{person}{Sharan Narang}, \bibinfo{person}{Jonah Alben},
  \bibinfo{person}{Gregory Diamos}, \bibinfo{person}{Erich Elsen},
  \bibinfo{person}{David Garcia}, \bibinfo{person}{Boris Ginsburg},
  \bibinfo{person}{Michael Houston}, \bibinfo{person}{Oleksii Kuchaiev},
  \bibinfo{person}{Ganesh Venkatesh}, {and} \bibinfo{person}{Hao Wu}.}
  \bibinfo{year}{2018}\natexlab{}.
\newblock \showarticletitle{Mixed Precision Training}. In
  \bibinfo{booktitle}{{\em International Conference on Learning
  Representations}}.
\newblock
\showURL{%
\url{https://openreview.net/forum?id=r1gs9JgRZ}}


\bibitem[\protect\citeauthoryear{Ott, Edunov, Grangier, and Auli}{Ott
  et~al\mbox{.}}{2018}]%
        {ott2018scaling}
\bibfield{author}{\bibinfo{person}{Myle Ott}, \bibinfo{person}{Sergey Edunov},
  \bibinfo{person}{David Grangier}, {and} \bibinfo{person}{Michael Auli}.}
  \bibinfo{year}{2018}\natexlab{}.
\newblock \showarticletitle{Scaling Neural Machine Translation}.
\newblock \bibinfo{journal}{{\em arXiv preprint arXiv:1806.00187\/}}
  (\bibinfo{year}{2018}).
\newblock


\bibitem[\protect\citeauthoryear{Polyak}{Polyak}{1964}]%
        {polyak1964mtm}
\bibfield{author}{\bibinfo{person}{Boris~T Polyak}.}
  \bibinfo{year}{1964}\natexlab{}.
\newblock \showarticletitle{Some methods of speeding up the convergence of
  iteration methods}.
\newblock \bibinfo{journal}{{\it U. S. S. R. Comput. Math. and Math. Phys.}}
  \bibinfo{volume}{4}, \bibinfo{number}{5} (\bibinfo{year}{1964}),
  \bibinfo{pages}{1--17}.
\newblock


\bibitem[\protect\citeauthoryear{Puri, Kirby, Yakovenko, and Catanzaro}{Puri
  et~al\mbox{.}}{2018}]%
        {puri2018largescale}
\bibfield{author}{\bibinfo{person}{Raul Puri}, \bibinfo{person}{Robert Kirby},
  \bibinfo{person}{Nikolai Yakovenko}, {and} \bibinfo{person}{Bryan
  Catanzaro}.} \bibinfo{year}{2018}\natexlab{}.
\newblock \showarticletitle{Large Scale Language Modeling: Converging on 40GB
  of Text in Four Hours}.
\newblock \bibinfo{journal}{{\em arXiv preprint arXiv:1808.01371\/}}
  (\bibinfo{year}{2018}).
\newblock


\bibitem[\protect\citeauthoryear{Radford, Narasimhan, Salimans, and
  Sutskever}{Radford et~al\mbox{.}}{2018}]%
        {radford2018improving}
\bibfield{author}{\bibinfo{person}{Alec Radford}, \bibinfo{person}{Karthik
  Narasimhan}, \bibinfo{person}{Tim Salimans}, {and} \bibinfo{person}{Ilya
  Sutskever}.} \bibinfo{year}{2018}\natexlab{}.
\newblock \showarticletitle{Improving language understanding by generative
  pre-training}.
\newblock \bibinfo{journal}{{\em Online\/}} (\bibinfo{year}{2018}).
\newblock


\bibitem[\protect\citeauthoryear{Schroff, Kalenichenko, and Philbin}{Schroff
  et~al\mbox{.}}{2015}]%
        {schroff2015facenet}
\bibfield{author}{\bibinfo{person}{Florian Schroff}, \bibinfo{person}{Dmitry
  Kalenichenko}, {and} \bibinfo{person}{James Philbin}.}
  \bibinfo{year}{2015}\natexlab{}.
\newblock \showarticletitle{Facenet: A unified embedding for face recognition
  and clustering}. In \bibinfo{booktitle}{{\em Proceedings of the IEEE
  conference on computer vision and pattern recognition}}.
  \bibinfo{pages}{815--823}.
\newblock


\bibitem[\protect\citeauthoryear{Shazeer, Mirhoseini, Maziarz, Davis, Le,
  Hinton, and Dean}{Shazeer et~al\mbox{.}}{2017}]%
        {shazeer2017moe}
\bibfield{author}{\bibinfo{person}{Noam Shazeer}, \bibinfo{person}{Azalia
  Mirhoseini}, \bibinfo{person}{Krzysztof Maziarz}, \bibinfo{person}{Andy
  Davis}, \bibinfo{person}{Quoc Le}, \bibinfo{person}{Geoffrey Hinton}, {and}
  \bibinfo{person}{Jeff Dean}.} \bibinfo{year}{2017}\natexlab{}.
\newblock \showarticletitle{Outrageously large neural networks: The
  sparsely-gated mixture-of-experts layer}.
\newblock \bibinfo{journal}{{\em arXiv preprint arXiv:1701.06538\/}}
  (\bibinfo{year}{2017}).
\newblock


\bibitem[\protect\citeauthoryear{Shazeer and Stern}{Shazeer and Stern}{2018}]%
        {adafactorICML2018}
\bibfield{author}{\bibinfo{person}{Noam Shazeer} {and}
  \bibinfo{person}{Mitchell Stern}.} \bibinfo{year}{2018}\natexlab{}.
\newblock \showarticletitle{Adafactor: Adaptive Learning Rates with Sublinear
  Memory Cost}. In \bibinfo{booktitle}{{\em Proceedings of the 35th
  International Conference on Machine Learning}} {\em
  (\bibinfo{series}{Proceedings of Machine Learning Research})},
  \bibfield{editor}{\bibinfo{person}{Jennifer Dy} {and}
  \bibinfo{person}{Andreas Krause}} (Eds.), Vol.~\bibinfo{volume}{80}.
  \bibinfo{publisher}{PMLR}, \bibinfo{address}{Stockholmsmässan, Stockholm
  Sweden}, \bibinfo{pages}{4596--4604}.
\newblock
\showURL{%
\url{http://proceedings.mlr.press/v80/shazeer18a.html}}


\bibitem[\protect\citeauthoryear{Shrivastava, Konig, and Bilenko}{Shrivastava
  et~al\mbox{.}}{2016}]%
        {shrivastava2016time}
\bibfield{author}{\bibinfo{person}{Anshumali Shrivastava},
  \bibinfo{person}{Arnd~Christian Konig}, {and} \bibinfo{person}{Mikhail
  Bilenko}.} \bibinfo{year}{2016}\natexlab{}.
\newblock \showarticletitle{Time adaptive sketches (ada-sketches) for
  summarizing data streams}. In \bibinfo{booktitle}{{\em Proceedings of the
  2016 International Conference on Management of Data}}. ACM,
  \bibinfo{pages}{1417--1432}.
\newblock


\bibitem[\protect\citeauthoryear{Shrivastava and Li}{Shrivastava and
  Li}{2014}]%
        {shrivastava2014asymmetric}
\bibfield{author}{\bibinfo{person}{Anshumali Shrivastava} {and}
  \bibinfo{person}{Ping Li}.} \bibinfo{year}{2014}\natexlab{}.
\newblock \showarticletitle{Asymmetric LSH (ALSH) for sublinear time maximum
  inner product search (MIPS)}. In \bibinfo{booktitle}{{\em Advances in Neural
  Information Processing Systems}}. \bibinfo{pages}{2321--2329}.
\newblock


\bibitem[\protect\citeauthoryear{Siskind and Pearlmutter}{Siskind and
  Pearlmutter}{2018}]%
        {siskind2018divide}
\bibfield{author}{\bibinfo{person}{Jeffrey~Mark Siskind} {and}
  \bibinfo{person}{Barak~A Pearlmutter}.} \bibinfo{year}{2018}\natexlab{}.
\newblock \showarticletitle{Divide-and-conquer checkpointing for arbitrary
  programs with no user annotation}.
\newblock \bibinfo{journal}{{\em Optimization Methods and Software\/}}
  \bibinfo{volume}{33}, \bibinfo{number}{4-6} (\bibinfo{year}{2018}),
  \bibinfo{pages}{1288--1330}.
\newblock


\bibitem[\protect\citeauthoryear{Spring and Shrivastava}{Spring and
  Shrivastava}{2017}]%
        {spring2017lsh}
\bibfield{author}{\bibinfo{person}{Ryan Spring} {and}
  \bibinfo{person}{Anshumali Shrivastava}.} \bibinfo{year}{2017}\natexlab{}.
\newblock \showarticletitle{Scalable and sustainable deep learning via
  randomized hashing}. In \bibinfo{booktitle}{{\em Proceedings of the 23rd ACM
  SIGKDD International Conference on Knowledge Discovery and Data Mining}}.
  ACM, \bibinfo{pages}{445--454}.
\newblock


\bibitem[\protect\citeauthoryear{Sutskever, Martens, Dahl, and
  Hinton}{Sutskever et~al\mbox{.}}{2013}]%
        {sutskever2013mtm}
\bibfield{author}{\bibinfo{person}{Ilya Sutskever}, \bibinfo{person}{James
  Martens}, \bibinfo{person}{George Dahl}, {and} \bibinfo{person}{Geoffrey
  Hinton}.} \bibinfo{year}{2013}\natexlab{}.
\newblock \showarticletitle{On the importance of initialization and momentum in
  deep learning}. In \bibinfo{booktitle}{{\em International conference on
  machine learning}}. \bibinfo{pages}{1139--1147}.
\newblock


\bibitem[\protect\citeauthoryear{Tai, Sharan, Bailis, and Valiant}{Tai
  et~al\mbox{.}}{2018}]%
        {sketchSIGMOD2018}
\bibfield{author}{\bibinfo{person}{Kai~Sheng Tai}, \bibinfo{person}{Vatsal
  Sharan}, \bibinfo{person}{Peter Bailis}, {and} \bibinfo{person}{Gregory
  Valiant}.} \bibinfo{year}{2018}\natexlab{}.
\newblock \showarticletitle{Sketching Linear Classifiers over Data Streams}. In
  \bibinfo{booktitle}{{\em Proceedings of the 2018 International Conference on
  Management of Data}} {\em (\bibinfo{series}{SIGMOD '18})}.
  \bibinfo{publisher}{ACM}, \bibinfo{address}{New York, NY, USA},
  \bibinfo{pages}{757--772}.
\newblock
\showISBNx{978-1-4503-4703-7}
\showDOI{%
\url{https://doi.org/10.1145/3183713.3196930}}


\bibitem[\protect\citeauthoryear{Tito~Svenstrup, Hansen, and
  Winther}{Tito~Svenstrup et~al\mbox{.}}{2017}]%
        {hashNIPS2017}
\bibfield{author}{\bibinfo{person}{Dan Tito~Svenstrup}, \bibinfo{person}{Jonas
  Hansen}, {and} \bibinfo{person}{Ole Winther}.}
  \bibinfo{year}{2017}\natexlab{}.
\newblock \showarticletitle{Hash Embeddings for Efficient Word
  Representations}.
\newblock In \bibinfo{booktitle}{{\em Advances in Neural Information Processing
  Systems 30}}, \bibfield{editor}{\bibinfo{person}{I.~Guyon},
  \bibinfo{person}{U.~V. Luxburg}, \bibinfo{person}{S.~Bengio},
  \bibinfo{person}{H.~Wallach}, \bibinfo{person}{R.~Fergus},
  \bibinfo{person}{S.~Vishwanathan}, {and} \bibinfo{person}{R.~Garnett}}
  (Eds.). \bibinfo{publisher}{Curran Associates, Inc.},
  \bibinfo{pages}{4928--4936}.
\newblock


\bibitem[\protect\citeauthoryear{Vaswani, Shazeer, Parmar, Uszkoreit, Jones,
  Gomez, Kaiser, and Polosukhin}{Vaswani et~al\mbox{.}}{2017}]%
        {attentionNIPS2017}
\bibfield{author}{\bibinfo{person}{Ashish Vaswani}, \bibinfo{person}{Noam
  Shazeer}, \bibinfo{person}{Niki Parmar}, \bibinfo{person}{Jakob Uszkoreit},
  \bibinfo{person}{Llion Jones}, \bibinfo{person}{Aidan~N Gomez},
  \bibinfo{person}{\L~ukasz Kaiser}, {and} \bibinfo{person}{Illia Polosukhin}.}
  \bibinfo{year}{2017}\natexlab{}.
\newblock \showarticletitle{Attention is All you Need}.
\newblock In \bibinfo{booktitle}{{\em Advances in Neural Information Processing
  Systems 30}}, \bibfield{editor}{\bibinfo{person}{I.~Guyon},
  \bibinfo{person}{U.~V. Luxburg}, \bibinfo{person}{S.~Bengio},
  \bibinfo{person}{H.~Wallach}, \bibinfo{person}{R.~Fergus},
  \bibinfo{person}{S.~Vishwanathan}, {and} \bibinfo{person}{R.~Garnett}}
  (Eds.). \bibinfo{publisher}{Curran Associates, Inc.},
  \bibinfo{pages}{5998--6008}.
\newblock
\showURL{%
\url{http://papers.nips.cc/paper/7181-attention-is-all-you-need.pdf}}


\bibitem[\protect\citeauthoryear{Vijayanarasimhan, Shlens, Monga, and
  Yagnik}{Vijayanarasimhan et~al\mbox{.}}{2014}]%
        {vijayanarasimhan2014deep}
\bibfield{author}{\bibinfo{person}{Sudheendra Vijayanarasimhan},
  \bibinfo{person}{Jonathon Shlens}, \bibinfo{person}{Rajat Monga}, {and}
  \bibinfo{person}{Jay Yagnik}.} \bibinfo{year}{2014}\natexlab{}.
\newblock \showarticletitle{Deep networks with large output spaces}.
\newblock \bibinfo{journal}{{\em arXiv preprint arXiv:1412.7479\/}}
  (\bibinfo{year}{2014}).
\newblock


\bibitem[\protect\citeauthoryear{Wen, He, Rajbhandari, Zhang, Wang, Liu, Hu,
  Chen, and Li}{Wen et~al\mbox{.}}{2017}]%
        {wen2017learning}
\bibfield{author}{\bibinfo{person}{Wei Wen}, \bibinfo{person}{Yuxiong He},
  \bibinfo{person}{Samyam Rajbhandari}, \bibinfo{person}{Minjia Zhang},
  \bibinfo{person}{Wenhan Wang}, \bibinfo{person}{Fang Liu},
  \bibinfo{person}{Bin Hu}, \bibinfo{person}{Yiran Chen}, {and}
  \bibinfo{person}{Hai Li}.} \bibinfo{year}{2017}\natexlab{}.
\newblock \showarticletitle{Learning intrinsic sparse structures within long
  short-term memory}.
\newblock \bibinfo{journal}{{\em arXiv preprint arXiv:1709.05027\/}}
  (\bibinfo{year}{2017}).
\newblock


\bibitem[\protect\citeauthoryear{Yang, Dai, Salakhutdinov, and Cohen}{Yang
  et~al\mbox{.}}{2018}]%
        {yang2018breaking}
\bibfield{author}{\bibinfo{person}{Zhilin Yang}, \bibinfo{person}{Zihang Dai},
  \bibinfo{person}{Ruslan Salakhutdinov}, {and} \bibinfo{person}{William~W.
  Cohen}.} \bibinfo{year}{2018}\natexlab{}.
\newblock \showarticletitle{Breaking the Softmax Bottleneck: A High-Rank {RNN}
  Language Model}. In \bibinfo{booktitle}{{\em International Conference on
  Learning Representations}}.
\newblock
\showURL{%
\url{https://openreview.net/forum?id=HkwZSG-CZ}}


\bibitem[\protect\citeauthoryear{Yen, Kale, Yu, Holtmann-Rice, Kumar, and
  Ravikumar}{Yen et~al\mbox{.}}{2018a}]%
        {lossICML2018}
\bibfield{author}{\bibinfo{person}{Ian En-Hsu Yen}, \bibinfo{person}{Satyen
  Kale}, \bibinfo{person}{Felix Yu}, \bibinfo{person}{Daniel Holtmann-Rice},
  \bibinfo{person}{Sanjiv Kumar}, {and} \bibinfo{person}{Pradeep Ravikumar}.}
  \bibinfo{year}{2018}\natexlab{a}.
\newblock \showarticletitle{Loss Decomposition for Fast Learning in Large
  Output Spaces}. In \bibinfo{booktitle}{{\em Proceedings of the 35th
  International Conference on Machine Learning}} {\em
  (\bibinfo{series}{Proceedings of Machine Learning Research})},
  \bibfield{editor}{\bibinfo{person}{Jennifer Dy} {and}
  \bibinfo{person}{Andreas Krause}} (Eds.), Vol.~\bibinfo{volume}{80}.
  \bibinfo{publisher}{PMLR}, \bibinfo{address}{Stockholmsmässan, Stockholm
  Sweden}, \bibinfo{pages}{5640--5649}.
\newblock
\showURL{%
\url{http://proceedings.mlr.press/v80/yen18a.html}}


\bibitem[\protect\citeauthoryear{Yen, Kale, Yu, Holtmann-Rice, Kumar, and
  Ravikumar}{Yen et~al\mbox{.}}{2018b}]%
        {pmlr-v80-yen18a}
\bibfield{author}{\bibinfo{person}{Ian En-Hsu Yen}, \bibinfo{person}{Satyen
  Kale}, \bibinfo{person}{Felix Yu}, \bibinfo{person}{Daniel Holtmann-Rice},
  \bibinfo{person}{Sanjiv Kumar}, {and} \bibinfo{person}{Pradeep Ravikumar}.}
  \bibinfo{year}{2018}\natexlab{b}.
\newblock \showarticletitle{Loss Decomposition for Fast Learning in Large
  Output Spaces}. In \bibinfo{booktitle}{{\em Proceedings of the 35th
  International Conference on Machine Learning}} {\em
  (\bibinfo{series}{Proceedings of Machine Learning Research})},
  \bibfield{editor}{\bibinfo{person}{Jennifer Dy} {and}
  \bibinfo{person}{Andreas Krause}} (Eds.), Vol.~\bibinfo{volume}{80}.
  \bibinfo{publisher}{PMLR}, \bibinfo{address}{Stockholmsmässan, Stockholm
  Sweden}, \bibinfo{pages}{5640--5649}.
\newblock


\bibitem[\protect\citeauthoryear{Zaheer, Reddi, Sachan, Kale, and Kumar}{Zaheer
  et~al\mbox{.}}{2018}]%
        {yogi_nips_2018}
\bibfield{author}{\bibinfo{person}{Manzil Zaheer}, \bibinfo{person}{Sashank
  Reddi}, \bibinfo{person}{Devendra Sachan}, \bibinfo{person}{Satyen Kale},
  {and} \bibinfo{person}{Sanjiv Kumar}.} \bibinfo{year}{2018}\natexlab{}.
\newblock \showarticletitle{Adaptive Methods for Nonconvex Optimization}.
\newblock In \bibinfo{booktitle}{{\em Advances in Neural Information Processing
  Systems 31}}, \bibfield{editor}{\bibinfo{person}{S.~Bengio},
  \bibinfo{person}{H.~Wallach}, \bibinfo{person}{H.~Larochelle},
  \bibinfo{person}{K.~Grauman}, \bibinfo{person}{N.~Cesa-Bianchi}, {and}
  \bibinfo{person}{R.~Garnett}} (Eds.). \bibinfo{publisher}{Curran Associates,
  Inc.}, \bibinfo{pages}{9815--9825}.
\newblock
\showURL{%
\url{http://papers.nips.cc/paper/8186-adaptive-methods-for-nonconvex-optimization.pdf}}


\end{thebibliography}

\onecolumn
\appendix
\section{Appendix}

{\bf Count-Sketch Error Bound:} 
\citep{charikar2002finding} Let $\hat{x}_i$ be the Count-Sketch estimate of component $i$ from vector $x$. For any component $x_i$, with probability $1-\delta$, a Count-Min Sketch matrix with width $\Theta(\frac{1}{\epsilon^2_1})$ and depth $\Theta(\log(\frac{d}{\delta}))$ satisfies
\begin{equation}
x_i - \epsilon_1 \norm{x}_2 \le \hat{x}_i \le x_i + \epsilon_1 \norm{x}_2
\end{equation}
{\bf Count-Min Sketch Error Bound:} 
\citep{cormode2005improved} Let $\hat{x}_i$ be the Count-Min Sketch estimate of component $i$ from vector $x$. For any component $x_i$, with probability $1-\delta$, a Count-Min Sketch matrix with width $\Theta(\frac{1}{\epsilon_1})$ and depth $\Theta(\log(\frac{d}{\delta}))$ satisfies
\begin{equation}
x_i \le \hat{x}_i \le x_i + \epsilon_1 \norm{x}_1
\end{equation}
For stochastic non-convex optimization, we measure how the algorithm converges to a stationary point - $\norm{\nabla f(x_t)}^2 \le c$ for some constant $c$. Notation: batch size $b$, learning rate $\eta_t$, 2nd moment decay rate $\beta_2$, count-min sketch error rate $\epsilon_1$, count-min sketch failure probability $\delta$.
\textbf{Assumptions:} Here are the assumptions used in our analysis:
\begin{enumerate}
    \item Function $f$ is L-Smooth - There exists a constant $L$ such that $\norm{\nabla f(x) - \nabla f(y)} \le L\norm{x - y},~\forall x,y \in \mathbb{R}^d$
    \item Function $f$ has bounded gradients - $[f(x_t)]_i \le G_i,~\forall x \in \mathbb{R}^d, i \in [d],~~~\hat{G} = \max_{i} G_i$
    \item The stochastic gradient oracle provides us with an unbiased estimate with fixed variance. Let $\xi_t$ represents the randomness (due to mini-batch sampling) at iteration $t$. $$g_{t,i} = [\nabla f(x_t, \xi_t)]_i,~~~\mathbb{E} [g_{t,i}] = [\nabla f(x_t)]_i,~~~\mathbb{E} [(g_{t,i} - [\nabla f(x_t)]_i)^2] \le \sigma_i$$
\end{enumerate}
% Count-Min Sketch Adam Update Rule
For simplicity and to save additional memory by not tracking the 1st moment, let $\beta_1=0$. In this form, the optimizer is commonly called RMSPROP. Therefore, the update rule for all $i \in [d]$ is 
\begin{equation}
x_{t+1,i} = x_{t,i} - \eta_t \frac{g_{t,i}}{\sqrt{\hat{v}_{t,i}} + \epsilon}, 
\end{equation}
where $\hat{v}_{t, i}$ represents the Count-Min Sketch estimate of component $i$ from vector $v_{t} = v_{t-1} + (1 - \beta_2)(v_{t-1} - g_t^2)$.
\begin{theorem} \label{eq:thm}
Let learning rate $\eta_t = \eta, \forall t \in [T]$ and batch size $b=1$. Assume $\beta_2$, $\eta$, and $\epsilon$ are selected such that $\eta \le \frac{\epsilon}{2L}$ and $\sqrt{1-\beta_2} \le \frac{\epsilon}{4\hat{G}}$. Given a Count-Min Sketch matrix width $\Theta(\frac{1}{\epsilon_1})$ and depth $\Theta(\log(\frac{dT}{\delta}))$, we have the following bound that holds for Count-Min Sketch Adam with probability $(1-\delta)$  where $M = (\frac{\hat{G} \sqrt{1-\beta_2}}{\epsilon^2} \sum^{d}_{i=1} \norm{G}^2_2)$:
\begin{align*}
\min_{t} \mathbb{E} \Bigg[ \norm{\nabla f(x_t)}^2 \Bigg]
\le
\mathcal{O} \Bigg ( \frac{f(x_0) - f(x_*)}{\eta T} + \sigma^2 + \epsilon_1 M \Bigg ) 
\end{align*}
\end{theorem}
\begin{proof}
% L-Smooth Assumption
% Insert Update Rule
Given that the function is $L$-smooth and by the optimizer update rule, we derive the following:
\begin{equation}
\begin{split}
f(x_{t+1}) &= f(x_t) + \langle \nabla f(x_t), x_{t+1} - x_t \rangle + \frac{L}{2} \norm{x_{t+1} - x_t}^2 \\
&= f(x_t) 
- \eta_t \sum^{d}_{i=1} \Bigg( [\nabla f(x_t)]_i \cdot \frac{g_{t,i}}{\sqrt{\hat{v}_{t,i}} + \epsilon} \Bigg) + \frac{L \eta^2_t}{2}  \sum^{d}_{i=1} \frac{g^2_{t,i}}{(\sqrt{\hat{v}_{t,i}} + \epsilon)^2}
\end{split}
\end{equation}
% Expectation of t+1 time step
Next, we take the expectation of $f(x_{t+1})$, given we that know $x_t$ (assumed fixed):
\begin{align*}
\mathbb{E}_t\left[f(x_{t+1})~|~x_t\right] &\le
f(x_t) 
- \eta_t \sum^{d}_{i=1} \Bigg( [\nabla f(x_t)]_i \cdot \mathbb{E}_t \Bigg[\frac{g_{t,i}}{\sqrt{\hat{v}_{t,i}} + \epsilon}~\Big|~x_t \Bigg] \Bigg) + \frac{L \eta^2_t}{2}  \sum^{d}_{i=1} \mathbb{E}_t \Bigg[\frac{g^2_{t,i}}{(\sqrt{\hat{v}_{t,i}} + \epsilon)^2}~\Big|~x_t\Bigg] \\
&=
f(x_t) 
-\eta_t \sum^{d}_{i=1} \Bigg( [\nabla f(x_t)]_i \cdot \mathbb{E}_t \Bigg[ \frac{g_{t,i}}{\sqrt{\hat{v}_{t,i}} + \epsilon} - \frac{g_{t,i}}{\sqrt{\beta_2 \hat{v}_{t-1,i}} + \epsilon} + \frac{g_{t,i}}{\sqrt{\beta_2 \hat{v}_{t-1,i}} + \epsilon} ~\Big|~x_t\Bigg] \Bigg) \\
&\quad \quad \quad ~~+ \frac{L \eta^2_t}{2}  \sum^{d}_{i=1} \mathbb{E}_t \Bigg[\frac{g^2_{t,i}}{(\sqrt{\hat{v}_{t,i}} + \epsilon)^2} ~\Big|~x_t\Bigg] \\
&=
f(x_t) 
-\eta_t \sum^{d}_{i=1} \Bigg( [\nabla f(x_t)]_i \cdot \Bigg[ \frac{[\nabla f(x_t)]_i}{\sqrt{\beta_2 \hat{v}_{t-1,i}} + \epsilon}
+ \mathbb{E}_t \Bigg[ \frac{g_{t,i}}{\sqrt{\beta_2 \hat{v}_{t,i}} + \epsilon} - \frac{g_{t,i}}{\sqrt{\beta_2 \hat{v}_{t-1,i}} + \epsilon} ~\Big|~x_t\Bigg] \Bigg] \Bigg) \\
&\quad \quad \quad ~~+ \frac{L \eta^2_t}{2}  \sum^{d}_{i=1} \mathbb{E}_t \Bigg[\frac{g^2_{t,i}}{(\sqrt{\hat{v}_{t,i}} + \epsilon)^2} ~\Big|~x_t\Bigg] \\
&\le
f(x_t) 
-\eta_t \sum^{d}_{i=1} \frac{[\nabla f(x_t)]^2_i}{\sqrt{\beta_2 \hat{v}_{t-1,i}} + \epsilon}
+\eta_t \sum^{d}_{i=1}  \Big | [\nabla f(x_t)]_i \Big | \cdot \Bigg| \mathbb{E}_t \Bigg[ \underbrace{\frac{g_{t,i}}{\sqrt{\beta_2 \hat{v}_{t,i}} + \epsilon} - \frac{g_{t,i}}{\sqrt{\beta_2 \hat{v}_{t-1,i}} + \epsilon}}_{T_1} ~\Big|~x_t\Bigg] \Bigg| \\
&\quad \quad \quad ~~+ \frac{L \eta^2_t}{2}  \sum^{d}_{i=1} \mathbb{E}_t \Bigg[\frac{g^2_{t,i}}{(\sqrt{\hat{v}_{t,i}} + \epsilon)^2} ~\Big|~x_t\Bigg]
\end{align*}
% Bound T1
The second equality occurs because $g_{t,i}$ is an unbiased estimate of $[\nabla f(x_t)]_i$. 
Now, we upper-bound the term $T_1$:
\begin{align*}
T_1 &= \frac{g_{t,i}}{\sqrt{\hat{v}_{t,i}} + \epsilon} - \frac{g_{t,i}}{\sqrt{\beta_2 \hat{v}_{t-1,i}} + \epsilon}
\\
&\le \abs{g_{t,i}} \cdot \Bigg| \frac{1}{\sqrt{\hat{v}_{t,i}} + \epsilon} - \frac{1}{\sqrt{\beta_2 \hat{v}_{t-1,i}} + \epsilon} \Bigg|
\\
&= \frac{\abs{g_{t,i}}}{(\sqrt{\hat{v}_{t,i}} + \epsilon) (\sqrt{\beta_2 \hat{v}_{t-1,i}} + \epsilon)} \cdot
\Bigg| \frac{\hat{v}_{t,i} - \beta_2 \hat{v}_{t-1,i}}{\sqrt{\hat{v}_{t,i}} + \sqrt{\beta_2 \hat{v}_{t-1,i}}} \Bigg|
\\
&= \frac{\abs{g_{t,i}}}{(\sqrt{\hat{v}_{t,i}} + \epsilon) (\sqrt{\beta_2 \hat{v}_{t-1,i}} + \epsilon)} \cdot
\Bigg| \frac{(1 - \beta_2)\hat{g}_{t,i}^2}{\sqrt{\hat{v}_{t,i}} + \sqrt{\beta_2 \hat{v}_{t-1,i}}} \Bigg|
\\
&\le \frac{\abs{g_{t,i}}}{(\sqrt{\hat{v}_{t,i}} + \epsilon) (\sqrt{\beta_2 \hat{v}_{t-1,i}} + \epsilon)} \cdot
\Bigg| \frac{(1-\beta_2)(g^2_{t,i} + \epsilon_1\norm{g^2_t}_1)}{\sqrt{\hat{v}_{t,i}} + \sqrt{\hat{v}_{t-1,i}}} \Bigg|
\\
&\le \frac{\sqrt{1-\beta_2}(g^2_{t,i} + \epsilon_1\norm{g^2_t}_1)}{ (\sqrt{\beta_2 \hat{v}_{t-1,i}} + \epsilon) \epsilon}
\end{align*}
From Lemma \ref{lem:A3}, we have the second equality. The second inequality occurs because of Lemma \ref{lem:A2}, which is derived using the Count-Min Sketch error bound. The third inequality occurs because $\frac{\abs{g_{t,i}}}{\sqrt{\hat{v}_{t,i}} + \sqrt{\hat{v}_{t-1,i}}} \le \frac{1}{\sqrt{1 - \beta_2}}$ and when we drop $\sqrt{\hat{v}_{t,i}}$ from $(\sqrt{\hat{v}_{t,i}} + \epsilon)$.

By substituting the upper-bound for $T_1$, we arrive at the following:
% Substitute Bounds T1
% Split Expectation
% Use Lemma 1
\begin{align*}
\mathbb{E}_t\left[f(x_{t+1})~|~x_t\right] &\le
f(x_t) 
-\eta_t \sum^{d}_{i=1} \frac{[\nabla f(x_t)]^2_i}{\sqrt{\beta_2 \hat{v}_{t-1,i}} + \epsilon}
+ \frac{\eta_t \hat{G} \sqrt{1-\beta_2}}{\epsilon} \sum^{d}_{i=1}  \Bigg( \mathbb{E}_t \Bigg[ \frac{ g^2_{t,i} + \epsilon_1\norm{g^2_t}_1 }{ \sqrt{\beta_2 \hat{v}_{t-1,i}} + \epsilon)} ~\Big|~x_t\Bigg] \Bigg)
\\
&\quad \quad \quad ~~+ \frac{L \eta^2_t}{2 \epsilon}  \sum^{d}_{i=1} \mathbb{E}_t \Bigg[\frac{g^2_{t,i}}{\sqrt{\hat{v}_{t,i}} + \epsilon}~\Big|~x_t \Bigg]
\\
&\le
f(x_t) 
-\eta_t \sum^{d}_{i=1} \frac{[\nabla f(x_t)]^2_i}{\sqrt{\beta_2 \hat{v}_{t-1,i}} + \epsilon} \\
&\quad \quad \quad ~~+ \frac{\eta_t \hat{G} \sqrt{1-\beta_2}}{\epsilon} \sum^{d}_{i=1}  \Bigg( \mathbb{E}_t \Bigg[ \frac{ g^2_{t,i} } { \sqrt{\beta_2 \hat{v}_{t-1,i}} + \epsilon)}~\Big|~x_t \Bigg] + \mathbb{E}_t \Bigg[ \frac{ \epsilon_1\norm{g^2_t}_1 }{ \sqrt{\beta_2 \hat{v}_{t-1,i}} + \epsilon)}~\Big|~x_t \Bigg] \Bigg) \\
&\quad \quad \quad ~~+ \frac{L \eta^2_t}{2 \epsilon} \sum^{d}_{i=1} \mathbb{E}_t \Bigg[\frac{g^2_{t,i}}{\sqrt{\beta_2  \hat{v}_{t-1,i}} + \epsilon} ~\Big|~x_t\Bigg]
\\
&\le
f(x_t)
- \Bigg (\eta_t - \frac{\eta_t \hat{G} \sqrt{1-\beta_2}}{\epsilon} - \frac{L \eta_t^2}{2 \epsilon} \Bigg) \sum^{d}_{i=1} \frac{[\nabla f(x_t)]^2_i}{\sqrt{\beta_2 \hat{v}_{t-1,i}} + \epsilon} \\
&\quad \quad \quad ~~+ \Bigg (\frac{\eta_t \hat{G} \sqrt{1-\beta_2}}{\epsilon} + \frac{L \eta_t^2}{2 \epsilon} \Bigg) \sum^{d}_{i=1} \frac{\sigma^2_i}{b (\sqrt{\beta_2  \hat{v}_{t-1,i}} + \epsilon)}
\\
&\quad \quad \quad ~~+  \frac{\eta_t \hat{G} \sqrt{1-\beta_2}}{\epsilon} \sum^{d}_{i=1} \frac{\epsilon_1 \left(\sum^{d}_{j=1} \frac{\sigma^2_j}{b} + [\nabla f(x_t)]^2_j\right)}{\sqrt{\beta_2 \hat{v}_{t-1,i}} + \epsilon}
\\
&\le
f(x_t)
- \Bigg (\eta_t - \frac{\eta_t \hat{G} \sqrt{1-\beta_2}}{\epsilon} - \frac{L \eta_t^2}{2 \epsilon} \Bigg) \sum^{d}_{i=1} \frac{[\nabla f(x_t)]^2_i}{\sqrt{\beta_2 \hat{v}_{t-1,i}} + \epsilon} \\
&\quad \quad \quad ~~+ \Bigg (\frac{\eta_t \hat{G} \sqrt{1-\beta_2}}{\epsilon} + \frac{L \eta_t^2}{2 \epsilon} \Bigg) \sum^{d}_{i=1} \frac{\sigma^2_i}{b (\sqrt{\beta_2 \hat{v}_{t-1,i}} + \epsilon)}
+  \frac{\eta_t \hat{G} \sqrt{1-\beta_2}}{\epsilon} \sum^{d}_{i=1} \frac{\epsilon_1 \left(\frac{\sigma^2}{b} + \norm{G}^2_2\right)}{\sqrt{\beta_2 \hat{v}_{t-1,i}} + \epsilon}
\end{align*}
% Assume Conditions and Simplify
% Use Upper-Bound on V_(t-1)
The first inequality follows because the function has bounded gradients - $[\nabla f(x_t)]_i \le \hat{G}$. Now, the second inequality holds because $\hat{v}_{t,i} \ge \beta_2 \hat{v}_{t-1,i}$. In addition, we split  the $g^2_{t,i}$ and $\epsilon_1 \norm{g^2_t}_1$ terms using the linearity of expectation. For the third inequality, we use the result and definitions in Lemma \ref{lem:A1}. From the specified parameters for $\eta_t$, $\beta_2$, and $\epsilon$, we assume the following conditions hold: $\frac{\hat{G} \sqrt{1-\beta_2}}{\epsilon} \le \frac{1}{4}$ and $\frac{L \eta_t}{2 \epsilon} \le \frac{1}{4}$.

\vspace{-0.1in}

\begin{align*}
\mathbb{E}_t\left[f(x_{t+1})~|~x_t\right] &\le
f(x_t)
- \frac{\eta_t}{2} \sum^{d}_{i=1} \frac{[\nabla f(x_t)]^2_i}{\sqrt{\beta_2 \hat{v}_{t-1,i}} + \epsilon} 
+ \Bigg (\frac{\eta_t \hat{G} \sqrt{1-\beta_2}}{\epsilon} + \frac{L \eta_t^2}{2 \epsilon} \Bigg) \sum^{d}_{i=1} \frac{\sigma_i^2}{b (\sqrt{\beta_2 \hat{v}_{t-1,i}} + \epsilon)} \\
&\quad \quad \quad ~~+ \frac{\eta_t \hat{G} \sqrt{1-\beta_2}}{\epsilon} \sum^{d}_{i=1} \frac{\epsilon_1 \left(\frac{\sigma^2}{b} + \norm{G}^2_2\right)}{\sqrt{\beta_2 \hat{v}_{t-1,i}} + \epsilon}
\\
&\le
f(x_t)
- \frac{\eta_t}{2 (\sqrt{\beta_2 \epsilon_1 d} \hat{G} + \epsilon)} \norm{\nabla f(x_t)}^2
+ \Bigg (\frac{\eta_t \hat{G} \sqrt{1-\beta_2}}{\epsilon^2} + \frac{L \eta_t^2}{2 \epsilon^2} \Bigg) \frac{\sigma^2}{b} \\
&\quad \quad \quad ~~+  \frac{\eta_t \hat{G} \sqrt{1-\beta_2}}{\epsilon^2} \left( \frac{\epsilon_1 d\sigma^2}{b} + \epsilon_1 \sum^{d}_{i=1} \norm{G}^2_2 \right)
\\
&= 
f(x_t)
- \frac{\eta_t}{2 (\sqrt{\beta_2 \epsilon_1 d}\hat{G} + \epsilon)} \norm{\nabla f(x_t)}^2
+ \Bigg (\frac{\eta_t \hat{G} \sqrt{1-\beta_2}}{\epsilon^2} (1 + \epsilon_1 d) + \frac{L \eta_t^2}{2 \epsilon^2} \Bigg) \frac{\sigma^2}{b} \\
&\quad \quad \quad ~~+ \epsilon_1\left( \frac{\eta_t \hat{G} \sqrt{1-\beta_2}}{\epsilon^2} \sum^{d}_{i=1} \norm{G}^2_2\right)
\end{align*}

For the standard optimizer, $0 \le v_{t-1,i} \le \hat{G}^2$. For the Count-Min Sketch approximation, $\norm{v_{t-1}}_1 = \sum^d_{i=1} \abs{v_{t-1,i}} \le \sum^d_{i=1} \hat{G}^2 = d \hat{G}^2$. Therefore, this inequality holds $0 \le v_{t-1,i} \le \hat{v}_{t-1,i} \le v_{t-1,i} + \epsilon_1 \norm{v_{t-1}}_1 \le \epsilon_1 d \hat{G}^2$. In addition, this corollary follows $\frac{1}{\beta_2 \sqrt{\epsilon_1 d}\hat{G} + \epsilon} \le \frac{1}{\epsilon}$. The second inequality follows given the two inequalities for the Count-Min Sketch approximation.

% Telescoping Sum
Now, we take a telescoping sum over all the iterations, and taking the full expectation:
\begin{align*}
\frac{\eta}{2 \left(\sqrt{\beta_2 \epsilon_1 d}\hat{G} + \epsilon\right)} &\sum^T_{t=0} \mathbb{E} \Bigg[ \norm{\nabla f(x_t)}^2 \Bigg] \\
& \le
f(x_0) - \mathbb{E}[f(x_{T+1})] + \Bigg (\frac{\eta \hat{G} \sqrt{1-\beta_2}}{\epsilon^2} (1 + \epsilon_1 d) + \frac{L \eta^2}{2 \epsilon^2} \Bigg) \frac{T \sigma^2}{b}
+ \epsilon_1 T \Bigg( \frac{\eta \hat{G} \sqrt{1-\beta_2}}{\epsilon^2}  \sum^{d}_{i=1} \norm{G}^2_2 \Bigg)
\end{align*}

% Optimal X star
% Multiply By Constant
Finally, given that $f(x_*) \le f(x_{T+1})$ and by multiplying the equation with $\frac{2 (\sqrt{\beta_2 \epsilon_1 d}\hat{G} + \epsilon)}{\eta T}$, we arrive at our final result.
\begin{align*}
\frac{1}{T} &\sum^T_{t=0} \mathbb{E} \Bigg[ \norm{\nabla f(x_t)}^2 \Bigg] \\
& \le
2 \left(\sqrt{\beta_2 \epsilon_1 d}\hat{G} + \epsilon\right) \cdot \left[
\frac{f(x_0) - f(x_*)}{\eta T}
+ \Bigg (\frac{\hat{G} \sqrt{1-\beta_2}}{\epsilon^2} (1 + \epsilon_1 d) + \frac{L \eta}{2 \epsilon^2} \Bigg) \frac{\sigma^2}{b}
+ \epsilon_1 \Bigg(\frac{\hat{G} \sqrt{1-\beta_2}}{\epsilon^2} \sum^{d}_{i=1} \norm{G}^2_2\Bigg)
\right]
\end{align*}
\end{proof}

\begin{lemma}{\label{lem:A1}}
\citep{yogi_nips_2018} For all $i \in [d]$ and for the iterates $x_t$ where $t \in [T]$ for Count-Min Sketch Adam, the following inequality holds: $$ \mathbb{E}[ g^2_{t,i} ] \le \frac{\sigma^2_i}{b} + [\nabla f(x_t)]_i $$
\end{lemma}

\begin{lemma}{\label{lem:A2}}
For all $i \in [d]$ and for the iterates $x_t$ where $t \in [T]$ for Count-Min Sketch Adam, the following inequality holds: $$\hat{v}_{t,i} - \beta_2 \hat{v}_{t,i}\le (1-\beta_2) (g^2_{t,i} + \norm{g_t}_1)$$
\end{lemma}

\begin{proof}
Given the error bound for the count-min sketch and the Adam update rule, we have the following:
\begin{align*}
    \hat{v}_{t,i} &\le v_{t,i} + \epsilon_1 \norm{v_t}_1 \\
    &= v_{t,i} + \epsilon_1 \sum^d_{i=1} \abs{v_t,i} \\
    &= \beta_2 v_{t,i} + (1-\beta_2) g^2_{t,i} + \epsilon_1 \sum^d_{i=1} \abs{\beta_2 v_{t,i} + (1-\beta_2) g^2_{t,i}} \\
    &\leq \beta_2 v_{t,i} + (1-\beta_2) g^2_{t,i} + \epsilon_1 \left(\sum^d_{i=1} \abs{\beta_2 v_{t,i}} + \sum^d_{i=1} \abs{(1-\beta_2) g^2_{t,i}}\right) \\
    &\leq \beta_2 v_{t,i} + (1-\beta_2) g^2_{t,i} + \epsilon_1 \beta_2 \norm{v_{t-1}}_1 + \epsilon_1 (1-\beta_2)\norm{g_t}_1
\end{align*}
By subtracting $\hat{v}_{t,i} - \beta_2 \hat{v}_{t,i}$ and simplifying, we derive the desired inequality.
\begin{align*}
    \hat{v}_{t,i} - \beta_2 \hat{v}_{t,i} &\le \beta_2 v_{t,i} + (1-\beta_2) g^2_{t,i} + \epsilon_1 \beta_2 \norm{v_{t-1}}_1 + \epsilon_1 (1-\beta_2)\norm{g_t}_1 - \beta_2 v_{t,i} - \beta_2 \epsilon_1 \norm{v_{t-1}}_1 \\
    &= (1-\beta_2) g^2_{t,i} + \epsilon_1 (1-\beta_2)\norm{g_t}_1 \\
    &= (1-\beta_2) (g^2_{t,i} + \norm{g_t}_1)
\end{align*}
\end{proof}

\begin{lemma}{\label{lem:A3}}
For all $i \in [d]$ and for the iterates $x_t$ where $t \in [T]$ for Count-Min Sketch Adam, the following equality holds:
\begin{align*}
\abs{g_{t,i}} \cdot \Bigg| \frac{1}{\sqrt{\hat{v}_{t,i}} + \epsilon} - \frac{1}{\sqrt{\beta_2 \hat{v}_{t-1,i}} + \epsilon} \Bigg|
= \frac{\abs{g_{t,i}}}{(\sqrt{\hat{v}_{t,i}} + \epsilon) (\sqrt{\beta_2 \hat{v}_{t-1,i}} + \epsilon)} \cdot
\Bigg| \frac{\hat{v}_{t,i} - \beta_2 \hat{v}_{t-1,i}}{\sqrt{\hat{v}_{t,i}} + \sqrt{\beta_2 \hat{v}_{t-1,i}}} \Bigg|
\end{align*}
\end{lemma}
\begin{proof}
Let $x = g_{t,i}$, $A = \sqrt{\hat{v}_{t,i}}$, and $B = \sqrt{\beta_2 \hat{v}_{t-1,i}}$
\begin{align*}
    \abs{x} \cdot \Bigg| \frac{1}{A + \epsilon} - \frac{1}{B + \epsilon} \Bigg| &= \abs{x} \cdot \Bigg| \frac{B-A}{(A + \epsilon)(B + \epsilon)} \Bigg| \\
    &= \abs{x} \cdot \Bigg| \frac{A-B}{(A + \epsilon)(B + \epsilon)} \Bigg|
    \\
    &= \frac{\abs{x} (A+B)}{(A+B)} \cdot \Bigg| \frac{A-B}{(A + \epsilon)(B + \epsilon)} \Bigg|
    \\
    &= \frac{\abs{x} (A+B)(A-B)}{(A+B)(A + \epsilon)(B + \epsilon)}
    \\
    &= \frac{\abs{x} (A^2 + B^2)}{(A+B)(A + \epsilon)(B + \epsilon)}
    \\
    &= \frac{\abs{x}}{(A + \epsilon)(B + \epsilon)} \cdot \Bigg| \frac{A^2 + B^2}{A + B} \Bigg|
\end{align*}
\end{proof}

\end{document}